\documentclass[sigconf]{acmart}
\AtBeginDocument{%
  }

\copyrightyear{2026}
\acmYear{2026}
\setcopyright{cc}
\setcctype{by}
\acmConference[KDD '26]{Proceedings of the 32nd ACM SIGKDD Conference on Knowledge Discovery and Data Mining V.2}{August 09--13, 2026}{Jeju Island, Republic of Korea}
\acmBooktitle{Proceedings of the 32nd ACM SIGKDD Conference on Knowledge Discovery and Data Mining V.2 (KDD '26), August 09--13, 2026, Jeju Island, Republic of Korea}
\acmDOI{10.1145/3770855.3817791}
\acmISBN{979-8-4007-2259-2/2026/08}

\usepackage{mathtools}
\usepackage{xcolor}
\newcommand{\bst}[1]{{\textbf{\textcolor{red}{#1}}}}
\newcommand{\subbst}[1]{\textcolor{blue}{\underline{#1}}}
\newcommand{\scalea}[1]{\scalebox{0.78}{#1}}
\newcommand{\scaleb}[1]{\scalebox{0.8}{#1}}
\usepackage{multirow}
\usepackage{tcolorbox}

\newcommand{\red}[1]{\textcolor[rgb]{1.00,0.00,0.00}{#1}}

\newcommand{\vecx}{\mathbf{x}}
\newcommand{\mbw}{\mathbf{w}}
\newcommand{\mbU}{\mathbf{U}}
\newcommand{\mbI}{\mathbf{I}}
\newcommand{\mbzero}{\mathbf{0}}

\newcommand{\mbf}{\mathbf{f}}
\newcommand{\mbg}{\mathbf{g}}

\newcommand{\mcN}{\mathcal{N}}
\newcommand{\mcX}{\mathcal{X}}

\newcommand{\mcR}{\mathcal{R}}
\newcommand{\mcI}{\mathcal{I}}
\newcommand{\mcE}{\mathcal{E}}
\newcommand{\mcF}{\mathcal{F}}
\newcommand{\mcW}{\mathcal{W}}
\newcommand{\mcS}{\mathcal{S}}
\newcommand{\mcL}{\mathcal{L}}
\newcommand{\mcO}{\mathcal{O}}

\newcommand{\beps}{\boldsymbol{\epsilon}}
\newcommand{\bmu}{\boldsymbol{\mu}}
\newcommand{\bSigma}{\boldsymbol{\Sigma}}
\newcommand{\bvareps}{\boldsymbol{\varepsilon}}
\newcommand{\bw}{\boldsymbol{w}}

\newcommand{\expt}{\mathbb{E}}
\newcommand{\N}{\mathbb{N}}
\newcommand{\R}{\mathbb{R}}
\newcommand{\C}{\mathbb{C}}

\newcommand{\diff}{\mathrm{d}}
\newcommand{\herm}{\mathrm{H}}

\DeclareMathOperator{\cov}{Cov}

\newcommand{\ie}{\textit{i}.\textit{e}., }
\newcommand{\eg}{\textit{e}.\textit{g}., }

\newcommand{\idots}{\mathinner{\rotatebox[origin=c]{45}{$\cdots$}}}

\theoremstyle{plain}
\newtheorem{theorem}{Theorem}[section]
\newtheorem{proposition}[theorem]{Proposition}
\newtheorem{lemma}[theorem]{Lemma}

\theoremstyle{definition}
\newtheorem{definition}[theorem]{Definition}

\theoremstyle{remark}

\usepackage{makecell}

\usepackage{tcolorbox}
\tcbset{
  tcbset/.style={
    colback=teal!10,
    colframe=teal!70,
    colbacktitle=teal!90,
    top=1pt, 
    bottom=1pt, 
    left=1pt, 
    right=1pt
  }
}

\begin{document}

% \title{Disentangled and Parallelized Frequency Diffusion for Efficient Time Series Generation}
\title{Parallel Complex Diffusion for Scalable Time Series Generation}

\author{Rongyao Cai}
\authornote{Both authors contributed equally to this research.}
\affiliation{
  \institution{Institute of Cyber-Systems and Control}
  \institution{Zhejiang University}
  \city{Hangzhou}
  \country{China}
}

\author{Yuxi Wan}
\authornotemark[1]
\affiliation{
  \institution{Institute of Cyber-Systems and Control}
  \institution{Zhejiang University}
  \city{Hangzhou}
  \country{China}
}

\author{Kexin Zhang}
\authornote{Corresponding authors. e-mail: \href{mailto:zhangkexin@zju.edu.cn}{\nolinkurl{zhangkexin@zju.edu.cn}}}
\affiliation{
  \institution{Institute of Cyber-Systems and Control}
  \institution{Zhejiang University}
  \city{Hangzhou}
  \country{China}
}

\author{Ming Jin}
\authornotemark[2]
\affiliation{
%   \institution{School of Information and Communication Technology}
  \institution{Griffith University}
  \city{Brisbane}
  \country{Australia}
}

\author{Zhiqiang Ge}
\affiliation{
    \institution{School of Mathematics}
    \institution{Southeast University}
    \city{Nanjing}
    \country{China}
}

\author{Qingsong Wen}
\affiliation{
  \institution{Squirrel Ai Learning}
  \city{Bellevue}
  \country{USA}
}

\author{Yong Liu}
\affiliation{
  \institution{Institute of Cyber-Systems and Control}
  \institution{Zhejiang University}
  \city{Hangzhou}
  \country{China}
}

\renewcommand{\shortauthors}{Rongyao Cai et al.}

\begin{abstract}
    Diffusion models learn data distributions indirectly through denoising, making the difficulty of generative modeling closely tied to the dependency structure of data. 
    For time series, strong temporal dependence forces the noise / score estimator to recover highly entangled cross-time relationships, leading to the \textit{curse of entanglement}. 
    We mitigate this burden by changing the topology of the diffusion space: the Discrete Fourier Transform (DFT) decomposes temporal dependencies into spectral modes, diagonalizing second-order dependency structure and better aligning the data manifold with isotropic Gaussian noise and homogeneous diffusion dynamics. 
    However, existing frequency-aware diffusion methods mainly use the DFT to design estimator blocks under temporal DDPM/SDE frameworks, while frequency-native diffusion paths face a mathematical barrier from complex-valued dynamics.
    We propose \textbf{PaCoDi} (\textbf{Pa}rallel \textbf{Co}mplex \textbf{Di}ffusion), a frequency-native diffusion framework that constructs the diffusion path in the spectral domain while replacing the complex-valued estimator with parallel real-valued estimators for real and imaginary components. 
    Theoretically, we prove the statistical orthogonality of spectral Gaussian noise, establish \textit{quadrature forward transitions} and \textit{conditional reverse factorization}, and extend discrete PaCoDi to continuous-time spectral SDEs through a \textit{Spectral Wiener Process}. 
    We further introduce a \textit{Mean Field Theory approximation} with an \textit{Interactive Correction Branch} to handle marginal coupling, and exploit Hermitian symmetry to reduce \textbf{50\%} attention FLOPs without information loss.
    Extensive experiments on unconditional and conditional time series generation demonstrate superior generative quality and computational efficiency against 5 SOTA baselines in 5 benchmarks, respectively. 
    Code is available at \href{https://github.com/RongyaoCai/PaCoDi}{\red{https://github.com/RongyaoCai/PaCoDi}}.
\end{abstract}

\begin{CCSXML}



\maketitle

\section{Introduction}

Diffusion models, including Denoising Diffusion Probabilistic Models (DDPMs)~\cite{NEURIPS2020_4c5bcfec} and Score-Based Generative Models (SGMs)~\cite{NEURIPS2020_92c3b916}, have emerged as a dominant paradigm in generative modeling. 
Recent unifying views further formalize these frameworks through the continuous-time lens of Stochastic Differential Equations (SDEs)~\cite{song2021scorebased, liu2025empoweringtsa}. 
At their core, these models indirectly learn data distributions by reversing a forward diffusion process, using a neural estimator to predict the score or noise at each diffusion step.
Consequently, the difficulty of estimation is intrinsically tied to the dependence structure of data: stronger structural dependence requires the estimator to capture highly coupled joint distributions across features.
This challenge is especially pronounced for time series, where temporal dependence is the defining characteristic that distinguishes them from sets of independent variables. 
An observation at a given temporal position is shaped by long-range dependencies, periodicity, and trends; thus, accurately estimating the noise or score at one point requires aggregating information across the entire sequence.
We refer to this dependency burden as the \textit{curse of entanglement}.

A natural strategy to mitigate this curse is to change the reparameterization space in which generative dynamics operate. 
Rather than forcing a denoising estimator to directly untangle dense temporal dependencies in the temporal domain, we apply the Discrete Fourier Transform (DFT) to map the sequence into the frequency domain. 
This spectral decomposition diagonalizes the second-order covariance structure, decoupling the entangled temporal variables into approximately orthogonal frequency modes.
Consequently, this reparameterization transforms the correlated data into a space that intrinsically aligns with Gaussian noise structure exposing isotropic diffusion transition.

The literature on diffusion models broadly bifurcates into two categories. 
The first establishes the \textit{fundamental mathematical frameworks} of the diffusion process, encompassing DDPMs~\cite{NEURIPS2020_4c5bcfec,song2021denoising}, score-based SDEs~\cite{song2021scorebased}, and probability-flow ODEs~\cite{lipman2023flow}. 
These works define the evolutionary dynamics of data within the state space.
The second focuses on the \textit{practical instantiation of the neural estimator}~\cite{liu2025empoweringtsa, peebles2023scalable, ijcai2025p580}, proposing advanced denoising algorithms tailored to these fixed diffusion paths. 
Despite these advancements, both lines predominantly perform diffusion in the native time or pixel domain. Consequently, they leave the deeply entangled coordinate representation of the time series fundamentally unchanged.

Most existing frequency-aware diffusion methods for time series remain confined to the second trajectory~\cite{10.1145/3783986,yuan2024diffusionts,Qian_2024_CVPR, Zhao_2024_CVPR}. 
They utilize the DFT to design specialized estimator blocks capable of mining global information within a pre-existing DDPM or SDE framework. 
Essentially, they modify only the neural estimator, not the underlying diffusion path. 
In this paradigm, the frequency domain is relegated to an auxiliary representation rather than serving as the native state space on which the generative process is defined.
At the level of the diffusion-path formulation, prior works~\cite{kim2025sequence,10.5555/3692070.3692444} represent rare attempts to construct a spectral diffusion path directly within the complex domain. 
However, simulating a valid complex diffusion process fundamentally necessitates a theoretically compatible score or velocity estimator. 
Frequency Diffusion~\cite{10.5555/3692070.3692444} circumvents this by flattening complex coefficients into concatenated real and imaginary parts, processing them through a \textit{monolithic real-valued backbone} (as Figure~\ref{fig:framework}a). 
This parameterization is theoretically conflicted: if interpreted as a native complex-valued estimator, it violates the \textit{holomorphic} (Cauchy-Riemann) constraints requisite for complex differentiability. 
Conversely, if interpreted as a real-valued mapping in $\mathbb{R}^{2L}$, it no longer explicitly conforms to the isomorphism with the underlying complex geometry induced by heteroscedastic spectral Gaussian noise. 
This disconnect exposes the central question addressed in this work: \textit{Can we exploit the DFT to simplify the dependency structure of the diffusion space while rigorously preserving the mathematical validity of complex SDE dynamics?}

\textit{Yes, we can!} 
To answer this, we propose PaCoDi (\textbf{Pa}rallel \textbf{Co}mplex \textbf{Di}ffusion), a frequency-native theoretical framework that defines the diffusion path in the spectral domain while resolving the estimator barrier of complex-valued dynamics through real-imaginary decoupled modeling.
Our formalism follows a rigorous ``Discrete-to-Continuous'' trajectory. 
We first prove the \textbf{statistical orthogonality} of the real and imaginary components of spectral Gaussian noise, which induces decoupled quadrature forward transitions in the frequency domain.
We then prove the \textbf{conditional factorization of reverse dynamics} given a fixed initial spectral state $\mcX_0$, yielding an idealized parallel reverse process for discrete DDPMs.
This factorization allows the complex-valued estimator to be reformulated as two parallel real-valued estimators for the real and imaginary components, enabling flexible neural architectures without violating complex differentiability constraints.
Since the true initial state is marginalized during generation, the marginal reverse dynamics inherently remain coupled by the data prior.
We therefore introduce \textit{Mean Field Theory} (MFT) to preserve the parallel architecture through an approximate factorization, accompanied by an \textit{interactive correction branch} to compensate for the phase coherence weakened by the MFT approximation.
Finally, we extend PaCoDi from discrete DDPM to a continuous-time formulation by defining a \textbf{Spectral Wiener Process} to govern heteroscedastic spectral Brownian motion.
From an efficiency perspective, PaCoDi further exploits \textit{Hermitian Symmetry}, reducing the computational complexity of nonlinear operators (\eg 50\% for attention blocks).

Our contributions establish a self-contained theoretical system:
\begin{itemize}
    \item \textbf{Rigorous Spectral Diffusion Theory:} 
    We establish the PaCoDi framework, characterized by quadrature forward transitions and conditional reverse factorization, both grounded in the statistical orthogonality of complex Gaussian noise.

    \item \textbf{A Unified Spectral SDE Framework:} 
    We unify discrete DDPMs and continuous SDEs by defining a spectral Wiener process to describe the heteroscedastic spectral Brownian motion. We further prove the equivalence between temporal, discrete and continuous frequency diffusions.

    \item \textbf{Approximate Factorization \& Correction:}
    We introduce an MFT approximation with an interactive correction branch to preserve parallel modeling while compensating for marginal coupling after the initial data state is marginalized.

    \item \textbf{Efficiency \& Empirical Performance:} 
    We exploit Hermitian symmetry to compress the spectral state space, reducing a 50\% computational complexity for attention block in PaCoDi. Extensive experiments on unconditional and conditional time series generation demonstrate consistently superior performance against 5 baselines, respectively.
\end{itemize}

\begin{figure*}
    \centering
    \includegraphics[width=0.90\linewidth]{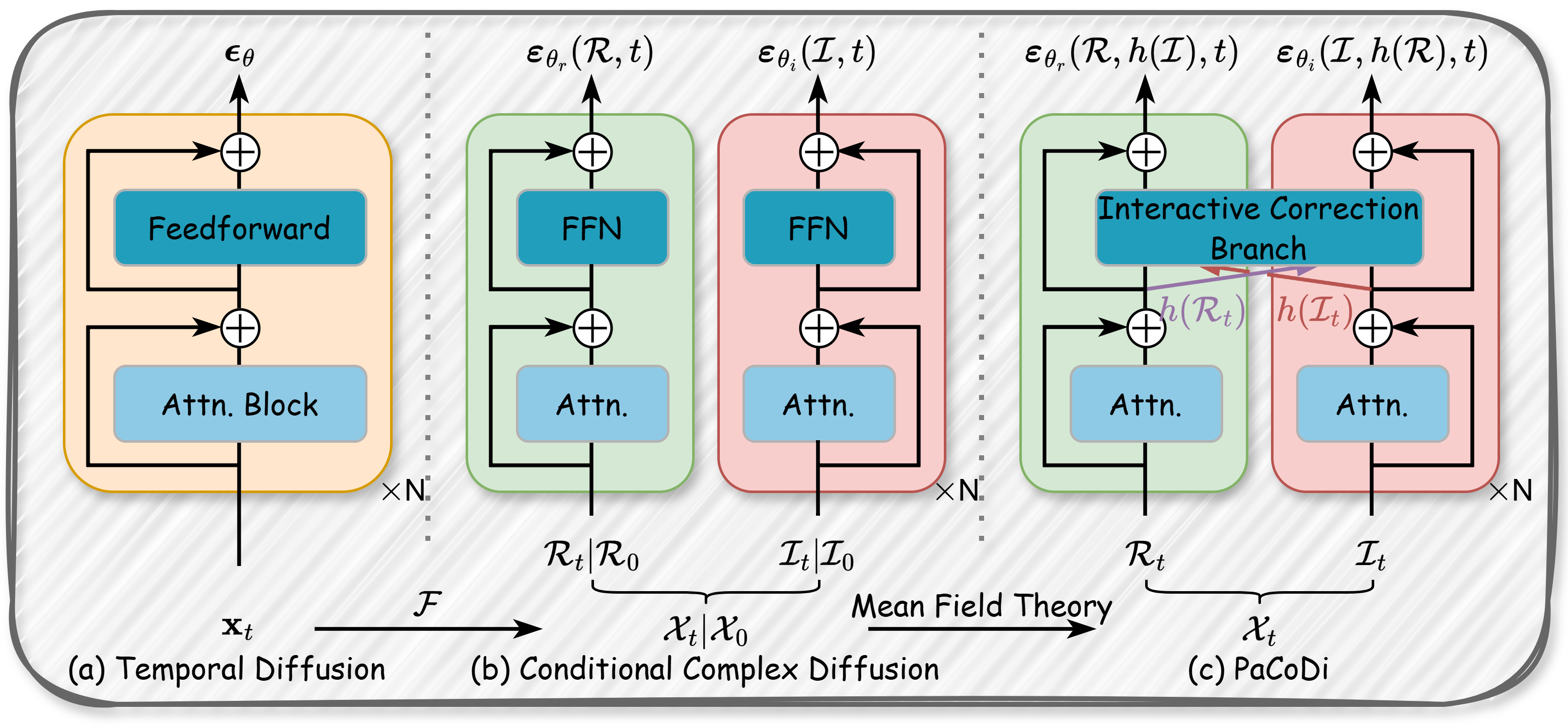}
    \vspace{-2ex}
    \caption{
        Architecture Evolution from Temporal Diffusion to PaCoDi. 
        DiT blocks are visualized in a simplified form (Attn. and FFN only) for brevity. 
        (a) Standard temporal diffusion. 
        (b) Conditional complex diffusion with quadrature components under fixed conditional state $\mcX_0$ in frequency manifold.
        (c) PaCoDi: Our proposed framework that utilizes an interactive correction branch to recover correlation constraints marginalized by the Mean Field Theory (MFT) approximation.
    }
    \label{fig:framework}
    \Description{
        Architecture evolution from temporal to parallel complex diffusion.
    }
\vspace{-2ex}
\end{figure*}

\section{Temporal Diffusion Modeling}
Diffusion models define a forward transition that transforms a complicated data distribution into a tractable prior. 
For clarity, we formulate the stochastic diffusion path along the temporal dimension $\vecx \in \R^L$; for multivariate time series, the channel dimension is preserved and jointly modeled by the neural estimator.

\textit{Forward Diffusion Transition.}\;
The diffusion transition corrupts the data $\vecx_0 \sim q(\vecx_0)$ into a sequence of latent states $\{ \vecx_t \}_{t=1}^T$. 
Each discrete diffusion step $t$ adds incremental Gaussian noise:
\begin{equation}
    q(\vecx_t | \vecx_{t-1})
    \coloneqq \mcN(\vecx_t; \sqrt{1 - \beta_t} \vecx_{t-1}, \beta_t \mbI),
    \label{eq:forward_transition}
\end{equation}
where $\beta_t \in (0, 1)$ denotes a pre-defined variance schedule and $\mbI$ is the identity covariance matrix. This iterative corruption ensures that for a sufficiently large $T$, the terminal state $\vecx_T$ converges to a standard isotropic Gaussian distribution.

\textit{Marginal Properties of the Temporal State Space.}\;
A fundamental analytical property of the linear Gaussian diffusion kernel is the closed-form expression of any intermediate state $\vecx_t$ conditioned directly on the initial observation $\vecx_0$.
\begin{proposition}
    \textbf{(Temporal Marginal Distribution).}
    Let $\alpha_t = 1 - \beta_t$ and $\bar{\alpha}_t = \prod_{s=1}^t \alpha_s$ be the cumulative signal retention factors. The marginal distribution of the state at step $t$ is given by:
    \begin{equation}
        q(\vecx_t | \vecx_0) = \mcN (\vecx_t; \sqrt{\bar{\alpha}_t} \vecx_0, (1-\bar{\alpha}_t) \mbI)
    \end{equation}
\end{proposition}

Through the reparameterization trick, the state $\vecx_t$ can be explicitly constructed with standard temporal Gaussian noise $\epsilon$:
\begin{equation}
    \vecx_t 
    = \sqrt{\bar{\alpha}_t} \vecx_0 + \sqrt{1 - \bar{\alpha}_t} \beps , \quad 
    \beps \sim \mcN(\mbzero, \mbI)
\end{equation}

While standard temporal diffusion effectively models sequence generation, the inherent \textit{curse of entanglement} between temporal positions forces the neural estimator to resolve highly coupled joint dependencies, bottlenecking optimization efficiency.

We address this bottleneck by shifting the diffusion dynamics into the complex frequency space. The Fourier transform reparameterizes temporal dependencies into spectral modes, and the induced spectral Gaussian noise admits an analytically tractable real-imaginary structure. This prepares the ground for a parallel complex diffusion process without claiming that real data spectra are unconditionally independent.

\section{Parallel Complex Diffusion Framework}

\subsection{Forward Diffusion in the Frequency Domain}
\label{sec:forward_diffusion}
We initiate our framework by projecting the discrete diffusion transition into the frequency domain.
We define the complex-valued state in the spectral domain as $\mcX_t = \mcF(\vecx_t) = \mcR_t + j \mcI_t \in \C^L$, where $\mcF(\cdot)$ denotes the DFT operator along the temporal dimension.

By the linearity of the Fourier transform, the Gaussian transition kernel of the standard DDPM, \ie Eq.~\eqref{eq:forward_transition}, induces an equivalent transition in the frequency domain:
\begin{equation}
    \mcX_t = \sqrt{1 - \beta_t} \mcX_{t-1} + \sqrt{\beta_t} \mcE_t ,
    \label{eq:forward_transition_freq}
\end{equation}
where $\mcE_t = \mcF(\beps_t)$ is the spectral projection of the standard temporal Gaussian noise $\beps_t \sim \mcN(\mbzero, \sigma^2 \mbI)$, which we define as \textit{Hermitian Complex Gaussian noise}. 
Theorem~\ref{thm:add_unitary_inv} ensures that weighted sums of such spectral Gaussian noises remain within the same Hermitian complex Gaussian family, enabling the standard DDPM reparameterization.
The tractability of parallelizing this diffusion process is fundamentally grounded in the statistical real-imaginary orthogonality and the heteroscedastic structure of $\mcE_t$.

\begin{tcolorbox}[tcbset]
\begin{theorem}
\label{thm:freq_ortho}
    \textbf{(Statistical Orthogonality of Complex Noise).}
    Let $\beps \sim \mcN(\mbzero, \sigma^2 \mbI)$ be an isotropic Gaussian vector. Its spectral projection $\mcE = \mcF(\beps)$ decomposes into real and imaginary components $\bvareps_r = \Re(\mcE)$ and $\bvareps_i = \Im(\mcE)$. These quadrature noise components are jointly Gaussian and cross-orthogonal:
    \begin{equation}
    \cov(\bvareps_r, \bvareps_i) = \mbzero_{L \times L}.
    \end{equation}
    where $\bvareps_r \sim \mcN(\mbzero, \bSigma_r)$ and $\bvareps_i \sim \mcN(\mbzero, \bSigma_i)$. Specifically, $\bSigma_r$ and $\bSigma_i$ are heteroscedastic covariance matrices determined by frequency indices and Nyquist frequency. 
\end{theorem}
\end{tcolorbox}

\textsc{Proof.}\; See Appendix~\ref{app:spec_dist_gauss}. \qed

Consequently, the induced Hermitian spectral noise law factorizes perfectly across its quadrature components:
\begin{equation}
    p(\mcE) \equiv p(\bvareps_r, \bvareps_i) = p(\bvareps_r) p(\bvareps_i) .
\end{equation}

By Theorem~\ref{thm:freq_ortho}, Eq.~\eqref{eq:forward_transition_freq} can be equivalently formulated as two parallel, real-valued transitions driven by independent spectral noise components. The conditional marginal distribution at any diffusion step $t$ can thus be expressed in closed form as:
\begin{equation}
    \mcR_t = \sqrt{\bar{\alpha}_t} \mcR_0 + \sqrt{1-\bar{\alpha}_t} \bvareps_r , \quad
    \mcI_t = \sqrt{\bar{\alpha}_t} \mcI_0 + \sqrt{1-\bar{\alpha}_t} \bvareps_i ,
\label{eq:decoupled_marginal}
\end{equation}
where $\bar{\alpha}_t = \prod_{s=1}^t (1-\beta_s)$ and $\bvareps_r \perp \bvareps_i$ are the cumulative spectral noise components from diffusion step 0 to $t$.
Crucially, this structural decoupling implies that $q(\mcX_t | \mcX_0) = q(\mcR_t | \mcR_0) q(\mcI_t | \mcI_0)$ (see Lemma~\ref{lem:cond_factor_traj}), allowing the parallel forward diffusion of real-imaginary modeling.

\subsection{Reverse Factorization Dynamics}
The generative foundation of our PaCoDi framework requires ensuring that the quadrature components remain structurally decoupled during the denoising phase. 
While the marginal distribution $p(\mcX_t)$ is inevitably coupled due to the complex correlations inherent in the natural data $\mcX_0$, the transitions conditioned on the initial data state exhibit a crucial independence property.

\begin{tcolorbox}[tcbset]
\begin{proposition}
\label{prop:cond_reverse_factor}
    \textbf{(Conditional Reverse Factorization).}
    Given the quadrature independence of the forward transition $q(\mcX_t | \mcX_{t-1})$, the reverse posterior $q(\mcX_{t-1} | \mcX_t, \mcX_0)$ factorizes into two independent marginals when conditioned on a fixed initial state $\mcX_0$:
    \begin{equation}
        q(\mcX_{t-1} | \mcX_t, \mcX_0)
        = q(\mcR_{t-1} | \mcR_t, \mcR_0) \cdot q(\mcI_{t-1} | \mcI_t, \mcI_0).
    \label{eq:reverse_factorization}
    \end{equation}
\end{proposition}
\end{tcolorbox}

\textsc{Proof.}\; See Appendix~\ref{app:cond_reverse_factor}. \qed

This proposition establishes the \textit{structural orthogonality} of the complex diffusion mechanism. It implies that if we model the reverse dynamics using the standard variational bound, the Evidence Lower Bound (ELBO) theoretically splits into two independent terms. 

However, unlike the standard temporal domain setting where noise is strictly isotropic, the complex noise components $\bvareps_r \sim \mcN(\mbzero, \bSigma_r)$ and $\bvareps_i \sim \mcN(\mbzero, \bSigma_i)$ are governed by specific covariance structures (\eg Hermitian symmetry). 
Thus, minimizing the ELBO corresponds to minimizing the Mahalanobis distance rather than the simple Euclidean norm (see Appendix~\ref{app:hetero_reverse_derivation}). 
Consequently, we can parameterize the noise estimators using two neural networks optimized via a decoupled objective:
\begin{equation}
\begin{aligned}
    &\min_\theta \expt_{\mcE, t | \mcX_0} 
    \left[ 
        \lambda_t \left \Vert 
            \mcE - \mcE_{\theta}(\mcX, t) 
        \right \Vert^2_{\bSigma^{-1}} 
    \right] \iff 
    \\
    &\min_{\theta_r, \theta_i} 
    \underbrace{ \expt_{\bvareps_r, t | \mcR_0} \left[ 
        \lambda_t \left \Vert 
            \bvareps_r - \bvareps_{\theta_r}(\mcR, t) 
        \right \Vert^2_{\bSigma_r^{-1}} 
    \right] }_{\mcL_{\text{real}}^{\text{disc}}(\theta_r)}
    + \underbrace{ \expt_{\bvareps_i, t | \mcI_0} \left[ 
        \lambda_t \left \Vert 
            \bvareps_i - \bvareps_{\theta_i}(\mcI, t) 
        \right \Vert^2_{\bSigma_i^{-1}} 
    \right] }_{\mcL_{\text{imag}}^{\text{disc}}(\theta_i)}
\end{aligned}
\label{eq:ideal_decoupled_loss}
\end{equation}
where $\Vert \mathbf{z} \Vert^2_{\bSigma^{-1}} = \mathbf{z}^\top \bSigma^{-1} \mathbf{z}$ denotes the squared Mahalanobis norm, and $\lambda_t = \frac{1 - \alpha_t}{2 \alpha_t (1 - \bar{\alpha}_t)}$ is the weighting term derived from the ELBO.

Eq.~\eqref{eq:ideal_decoupled_loss} presents an \textbf{idealized parallel diffusion framework}: optimizing the real trajectory $\mcL_{\text{real}}^{\text{disc}}(\theta_r)$ and the imaginary trajectory $\mcL_{\text{imag}}^{\text{disc}}(\theta_i)$ are mathematically separable tasks conditioned on the known initial state $\mcX_0$. 
Specifically, the real-part network $\theta_r$ takes only $\mcR_t$ as input and is constrained solely by the real noise component $\bvareps_r$, and likewise for the imaginary part.

\subsection{Parallel Modeling via Mean Field Theory}
\label{sec:mft_parallel}
While the \textit{true conditional posterior} $q(\mcX_{t-1} | \mcX_t, \mcX_0)$ is strictly factorized, the \textit{generative transition} required for sampling, $p_{\theta}(\mcX_{t-1} | \mcX_t)$, must approximate the intractable marginal $q(\mcX_{t-1} | \mcX_t)$. 
Because $q(\mcX_{t-1} | \mcX_t)$ integrates over the coupled data prior $q(\mcX_0)$, it is inherently entangled (\eg phase coherence in the signal):
\begin{equation} 
    q(\mcX_{t-1} | \mcX_t) 
    = \int q(\mcX_{t-1} | \mcX_t, \mcX_0) q(\mcX_0 | \mcX_t) \diff \mcX_0 .
\end{equation}

To reconcile our parallel modeling objective with these statistically coupled marginal dynamics, we invoke the \textbf{Mean Field Theory (MFT) approximation}. 
Instead of assuming that the true data prior are perfectly decoupled, we introduce a factorized variational proxy for the marginal reverse transition (Details in Appendix~\ref{sec:transition_to_mft_proof}):
\begin{equation}
    p_{\theta}(\mcX_{t-1} | \mcX_t) 
    \approx p_{\theta_r}(\mcR_{t-1} | \mcR_t) \cdot p_{\theta_i}(\mcI_{t-1} | \mcI_t) .
\label{eq:mft_approx}
\end{equation}

This factorized proxy allows the reverse estimator to be trained efficiently through two branch-wise objectives.

However, the naive mean-field factorization in Eq.~\eqref{eq:mft_approx} weakens the cross-quadrature dependencies required to preserve phase coherence. 
To compensate for this approximation, we introduce an \textbf{interactive correction branch} in the estimator parameterization.
Each branch predicts its own quadrature noise while receiving a projected representation of the counterpart quadrature:
\begin{equation}
    \mcE_{\theta}(\mcX, t)
    = \bvareps_{\theta_r}(\mcR, h(\mcI), t)
    + j \bvareps_{\theta_i}(\mcI, h(\mcR), t) ,
\end{equation}
where $h(\cdot)$ denotes a lightweight projection module that maps the counterpart quadrature into the branch's native input space.
In this design (illustrated in Figure~\ref{fig:framework}c), $p_{\theta}$ retains branch-wise parallel outputs while allowing cross-quadrature information flow.

By combining the Mahalanobis distance metric derived in Eq.~\eqref{eq:ideal_decoupled_loss} with this interactive parameterization, we arrive at the final parallel training objective for PaCoDi:
\begin{align} 
    \mcL_{\text{real}}^{\text{disc}} (\theta_r)
    &= \expt_{\mcX_0, \mcE, t} \left[ 
        \lambda_t \left \Vert 
            \bvareps_{\theta_r} (\mcR, h(\mcI), t) - \bvareps_r 
        \right \Vert^2_{\bSigma_r^{-1}} 
    \right] , 
\label{eq:final_loss_real} 
    \\ 
    \mcL_{\text{imag}}^{\text{disc}} (\theta_i)
    &= \expt_{\mcX_0, \mcE, t} \left[ 
        \lambda_t \left \Vert 
            \bvareps_{\theta_i} (\mcI, h(\mcR), t) - \bvareps_i 
        \right \Vert^2_{\bSigma_i^{-1}} 
    \right] .
\label{eq:final_loss_imag} 
\end{align}
The corresponding marginal score approximation and the formal justification for this loss factorization are provided in Appendix~\ref{app:mft_marginal}.

\subsection{Bypassing Holomorphic Constraints} 
Our framework resolves a critical theoretical conflict inherent in complex-valued deep learning: \textit{Holomorphicity vs. Expressivity}.

If the noise estimator is modeled as a native complex-valued neural function, complex differentiability requires satisfying the \textit{Cauchy-Riemann Equations} ($\frac{\partial u}{\partial x} = \frac{\partial v}{\partial y}, \frac{\partial u}{\partial y} = - \frac{\partial v}{\partial x}$). 
This holomorphic requirement severely restricts admissible nonlinear parameterizations. 
In particular, by \textit{Liouville's Theorem}, a bounded entire holomorphic function must be constant, revealing a fundamental tension between stability-oriented boundedness and nonlinear expressivity in native holomorphic modeling.

PaCoDi avoids this tension by keeping the diffusion path directly in the frequency domain, while reformulating the learnable reverse estimator into parallel real and imaginary quadrature branches justified by the conditional factorization. 
This real-imaginary decoupling rigorously satisfies the mathematical requirements of the complex diffusion dynamics, without forcing the neural estimator itself to be a holomorphic function. 
Consequently, each branch can flexibly employ standard, non-holomorphic activation functions, while the Mahalanobis objective and interactive correction branch preserve the statistical structure of the complex spectral process.

\section{Continuous-time Perspective: Frequency SDEs} 
\label{sec:continuous_sde}

To provide a unified theoretical framework for the discrete PaCoDi, we extend the frequency diffusion process to continuous-time SDEs.

\subsection{Forward SDE and Spectral Wiener Process}
Consider the discrete transition rule $\mcX_{t} = \sqrt{1 - \beta_t} \mcX_{t-1} + \sqrt{\beta_t} \mcE$. 
By taking the continuum limit $\Delta t \to 0$ and applying a Taylor expansion to the drift and diffusion coefficients (detailed derivations in Appendix~\ref{app:derivation_SDE}), the infinitesimal evolution of the frequency state converges to the following It\^o SDE:
\begin{equation}
    \diff \mcX
    = -\frac{1}{2}\beta(t)\mcX \diff t + \sqrt{\beta(t)} \diff \mcW
    = \mbf(\mcX, t) \diff t + \mbg(t) \diff \mcW ,
\label{eq:forward_SDE}
\end{equation}
where $\mbf(\cdot)$ and $\mbg(\cdot)$ denote the drift and diffusion coefficients of a Variance Preserving (VP) SDE. 
The stochasticity of continuous-time PaCoDi is driven by newly defined \textit{Spectral Wiener Process} $\mcW$.

\begin{tcolorbox}[tcbset]
\begin{definition}
    \textbf{(Spectral Wiener Process).}
    Let $\mbw_t \in \R^L$ be a standard temporal Wiener process. 
    The \textit{Spectral Wiener Process} $\mcW_t$ is defined as the complex-valued process induced by the Fourier transform of $\mbw_t$:
    \begin{equation}
        \mcW_t \coloneqq \mcF(\mbw_t).
    \end{equation}
    
    Its differential increment, $\diff \mcW = \diff \bw_r + j \diff \bw_i$, is characterized by the following spectral properties:
    (1) \textit{Quadrature Orthogonality}: The real and imaginary components are strictly statistically independent, \ie $\diff \bw_r \perp \diff \bw_i$.
    (2) \textit{Complex Gaussianity}: The quadrature increments $\diff \bw_r$ and $\diff \bw_i$ are Gaussian with heteroscedastic covariance matrices $\bSigma_r \diff t$ and $\bSigma_i \diff t$, respectively.
\end{definition}
\end{tcolorbox}

\subsection{Bridging Complex Scores and Parallel Real-valued SDEs}
To enable generation, we must reverse this continuous diffusion process. 
The reverse-time dynamics are governed by the gradient of the log-probability density, known as the \textit{score function}.

\subsubsection{Ideal Quadrature Reverse Dynamics}
We first establish the theoretical form of the reverse SDE under ideal conditions. 
Utilizing \textit{Wirtinger Calculus} for non-holomorphic functions, the complex score function decomposes linearly:
\begin{equation}
    \nabla_{\bar{\mcX}} \log p_{t|0}(\mcX)
    = \frac{1}{2} \left(
        \nabla_{\mcR} \log p_{t|0}(\mcX) 
        + j \nabla_{\mcI} \log p_{t|0}(\mcX)
    \right) ,
\label{eq:score_decomp}
\end{equation}
where $p_{t|0}(\mcX) \coloneqq p(\mcX_t | \mcX_0)$ is the true conditional transition density.

Leveraging the conditional factorizability (Lemma~\ref{lem:cond_factor_traj}), for a fixed initial state $\mcX_0$, the cross-terms in the log-density vanish, allowing the complex score to strictly split into independent real and imaginary gradients, \ie $\nabla_{\mcR} \log p_{t|0}(\mcX) = \nabla_{\mcR} \log p_{t|0}(\mcR)$. 

\begin{tcolorbox}[tcbset]
\begin{theorem}
\label{thm:reverse_SDE}
    \textbf{(Parallel Quadrature Reverse SDE).}
    Under the quadrature independence of the Spectral Wiener Process, the reverse-time SDE for a fixed initial condition $\mcX_0$ is given by:
    \begin{equation}
        \diff \mcX
        = \left[ 
            \mbf(\mcX, t) - \mbg(t) \bSigma_{\mcX} \mbg(t)^{\top} \nabla_{\bar{\mcX}} \log p_{t|0}(\mcX) 
        \right] \diff t + \mbg(t) \diff \hat{\mcW} ,
    \label{eq:reverse_complex_SDE}
    \end{equation}
    which naturally separates into two parallel, real-valued SDEs:
    \begin{equation}
    \begin{dcases}
        \diff \mcR
        = \left[
            \mbf(\mcR, t) - \frac{1}{2} \mbg(t) \bSigma_r \mbg(t)^{\top} \nabla_{\mcR} \log p_{t|0}(\mcR)
        \right] \diff t + \mbg(t) \diff \hat{\bw}_r
        \\
        \diff \mcI
        = \left[
            \mbf(\mcI, t) - \frac{1}{2} \mbg(t) \bSigma_i \mbg(t)^{\top} \nabla_{\mcI} \log p_{t|0}(\mcI)
        \right] \diff t + \mbg(t) \diff \hat{\bw}_i
    \end{dcases}
    \end{equation}
    where $\diff \hat{\mcW} = \diff \hat{\bw}_{r} + j \diff \hat{\bw}_{i}$ is the increment of the inverse Spectral Wiener Process flowing backward in time.
\end{theorem}
\end{tcolorbox}

\textsc{Proof.}\; See Appendix~\ref{app:reverse_SDE}. \qed

\subsubsection{MFT Approximation and Interactive Score Matching}
While Theorem~\ref{thm:reverse_SDE} guarantees decoupling for the conditional distribution $p_{t|0}(\mcX)$, the practical generative process relies on the marginal score $\nabla_{\bar{\mcX}} \log p_t(\mcX)$, which integrates over the data distribution $p(\mcX_0)$. 
Since natural signals exhibit strong phase coherence, $p(\mcX_0)$ is inherently coupled, rendering the true marginal score entangled. 

To resolve the conflict between the decoupled SDE theory and the coupled data reality, we apply the MFT approximation (as motivated in Section~\ref{sec:mft_parallel}) to the continuous-time setting. 
We approximate the marginal score via the interactive correction branch estimator $\mcS_{\theta}$:
\begin{equation}
    \nabla_{\bar{\mcX}} \log p_t(\mcX) 
    \approx \mcS_{\theta}(\mcX, t)
    = \frac{1}{2} \left(
        s_{\theta_r}(\mcR, h(\mcI), t) + j s_{\theta_i}(\mcI, h(\mcR), t)
    \right) .
\end{equation}

This factorization allows us to train the continuous-time model via Denoising Score Matching. 
Crucially, due to the spectral heteroscedasticity of complex noise, we derive the \textit{Score-Noise Identity} (Appendix~\ref{app:der_heter_score_noise_id}) which relates the score to the cumulative injected noise $\bvareps_r$ and $\bvareps_i$ from diffusion step 0 to $t$ via the precision matrix:
\begin{equation}
    \nabla_{\mcR} \log p_{t|0}(\mcR)
    = - \frac{\bSigma_r^{-1} \bvareps_r}{\sqrt{1 - \bar{\alpha}_t}} , \quad
    \nabla_{\mcI} \log p_{t|0}(\mcI)
    = - \frac{\bSigma_i^{-1} \bvareps_i}{\sqrt{1 - \bar{\alpha}_t}} .
\end{equation}

Substituting this identity into the Fisher Divergence objective, we obtain the final continuous-time training objective, which mirrors the discrete case while providing rigorous SDE grounding:
\begin{align}
    \mcL_{\text{real}}^{\text{cont}}(\theta_r)
    &= \expt_{\mcX_0, \mcE, t} \left[
        \lambda(t) \left \Vert
            s_{\theta_r}(\mcR, h(\mcI), t) - \nabla_{\mcR} \log p_{t}(\mcR)
        \right \Vert_{\bSigma_r}^2
    \right] ,
    \\
    \mcL_{\text{imag}}^{\text{cont}}(\theta_i)
    &= \expt_{\mcX_0, \mcE, t} \left[
        \lambda(t) \left \Vert
            s_{\theta_i}(\mcI, h(\mcR), t) - \nabla_{\mcI} \log p_{t}(\mcI)
        \right \Vert_{\bSigma_i}^2
    \right] ,
\end{align}
where $\lambda(t) = \frac{1 - \alpha_t}{2 \alpha_t}$ is the standard temporal weighting term.

This formulation provides a continuous-time training objective consistent with the discrete PaCoDi architecture, while mathematically respecting the intrinsic geometry of the complex diffusion process.
We prove the mathematical equivalence between the discrete-time and continuous-time loss functions in Appendix~\ref{app:equiv_loss}.

\section{Spectral-Native Optimization and Efficiency} 
\label{sec:spectral_opt}

Shifting the diffusion process to the frequency domain fundamentally transforms the problem topology from a locally entangled temporal sequence to a globally decoupled spectral representation. 
We analyze this structural advantage from two perspectives: \textit{Topological Simplification} and \textit{Optimization Stability}.

\subsection{Topological Simplification \& Global Context}

\textbf{Orthogonal Decorrelation.} 
Standard temporal modeling is intrinsically burdened by the \textit{entanglement} between temporal positions. 
In contrast, the DFT acts as a natural \textit{diagonalizing operator}, where the second-order covariance structure is natively diagonalized for stationary processes.
This effectively decomposes the intractable joint distribution $p(\vecx_{1:L})$ into mathematically simpler, nearly independent sub-problems $p(\mcX_k)$, significantly lowering the functional complexity the neural estimator must approximate.

\textbf{Global Receptive Field by Design.} 
Unlike temporal convolutions or attention mechanisms that rely on deep stacking to expand their receptive field, spectral features are linear combinations of the \textit{entire} sequence. 
This grants PaCoDi a \textit{global receptive field} at the very first layer, enabling the direct capture of global structural patterns regardless of sequence length, making the model remarkably efficient at long-range dependency modeling.

\subsection{Unitary Isometry and Training Stability}
The normalized DFT is a unitary linear operator, satisfying $\mbU \mbU^{\herm} = \mbI$, imparting two critical properties that ensure training robustness:
\begin{itemize}
    \item \textbf{Loss Unbiasedness (Parseval's Equivalence):}
    Because $\mbU$ is an isometry, Euclidean distance is strictly preserved: $\Vert \vecx - \hat{\vecx} \Vert_2^2 = \Vert \mbU(\vecx - \hat{\vecx}) \Vert_2^2 = \Vert \mcX - \hat{\mcX} \Vert_2^2$. This guarantees that minimizing the ELBO in the frequency domain is mathematically identical to minimizing it in the temporal domain.

    \item \textbf{Optimization Stability:} 
    Unitary transformations preserve the spectral radius of the Hessian matrix. Consequently, PaCoDi inherits the well-studied optimization landscape of standard diffusion models, ensuring stable convergence without recalibrating the noise schedule $\beta_t$ or learning rates.
\end{itemize}

\subsection{Hermitian Symmetry and Compression}
While unitarity ensures stability, the Hermitian Symmetry of real-valued signals enables significant state-space compression.

For any real vector $\vecx \in \R^L$, its spectrum satisfies $\mcX_k = \bar{\mcX}_{L-k}$. 
By discarding the redundant negative frequencies, we define the \textbf{Compressed Frequency State}:
\begin{equation}
    \tilde{\mcX} = \left[ \mcX_0, \dots, \mcX_K \right]^{\top} \in \C^K , \quad 
    \text{where } K = \lfloor L/2 \rfloor .
\end{equation}

This reduction halves the modeling dimensionality. 
However, to maintain mathematical rigor in reverse process, we must account for the deterministic shift in noise distribution on compressed space:
\begin{equation}
    \tilde{\bvareps}_{r, k} 
    \sim \mcN \left(0, \frac{\sigma^2}{2} (1 + \delta_{k, L/2})\right) , \quad
    \tilde{\bvareps}_{i, k} 
    \sim \mcN \left(0, \frac{\sigma^2}{2} (1 - \delta_{k, L/2})\right) .
\end{equation}

\begin{table}[t]
    \centering
    \caption{
        Complexity analysis comparison between temporal DiT and PaCoDi. $L$ denotes sequence length.
    }
    \label{tab:complexity}
    \vspace{-2ex}
    \resizebox{1.0\linewidth}{!}{
        \begin{tabular}{l|cc|c}
        \toprule
        \textbf{Layer Type} & \textbf{Temporal DiT ($L$)} & \textbf{PaCoDi ($2 \times L/2$)} & \textbf{Speedup Factor} \\
        \midrule
        \textbf{Linear / MLP} 
        & $1 \times \mcO(L)$   & $2 \times \mcO(L/2)$      & \textbf{1.0$\times$ (Iso-FLOPs)} \\
        \textbf{Self-Attention} 
        & $1 \times \mcO(L^2)$ & $2 \times \mcO(L^2/4)$    & \textbf{2.0$\times$ (50\% Savings)} \\
        \textbf{FFT Overhead} 
        & N/A                  & $2 \times \mcO(L \log L)$ & N/A \\
        \bottomrule
    \end{tabular}
    }
\vspace{-10pt}
\end{table}

\subsection{Computational Acceleration: 50\% FLOPs}
By combining the state compression ($L \to L/2$) with our decoupled parallel architecture, PaCoDi achieves a massive 50\% reduction in attention FLOPs (\eg DiT~\cite{peebles2023scalable}) compared to temporal ones. 
Notably, the $\mcO(L \log L)$ cost of FFT/iFFT is asymptotically negligible.

\textbf{1. Linear Layers (Iso-FLOPs):}
For linear projections (MLP, Conv), computational complexity scales linearly with sequence length. Although we use two branches, each processes half the sequence length: $\text{Cost} \propto 2 \times (L/2) = L$.
Thus, PaCoDi maintains the same parameter count and FLOPs as the temporal baseline.

\textbf{2. Self-Attention (Quadratic Reduction):}
The bottleneck of DiT is the attention mechanism, which scales quadratically with sequence length, $\mcO(L^2)$. 
The reduction in effective sequence length triggers a nonlinear efficiency gain. Ultimately, PaCoDi reduces the total Attention FLOPs by \textbf{50\%}. 
Unlike naive complex-valued networks (which require 4 real multiplications per complex operation and yield no speedup), real-valued parallel design fully unlocks the computational potential of spectral compression:
\begin{align}
    \text{Temporal Attn.} 
    &= \mcO(L^2 \cdot C) 
    \\
    \text{PaCoDi Attn.} 
    &= \underbrace{\mcO((\frac{L}{2})^2 \cdot C)}_{\text{Real Branch}} 
    + \underbrace{\mcO((\frac{L}{2})^2 \cdot C)}_{\text{Imag Branch}} 
    = \frac{1}{2} \mcO(L^2 \cdot C).
\end{align}

\section{Experiments}

\subsection{Datasets \& Baselines}

To empirically validate PaCoDi, we evaluate its performance against SOTA methods across conditional and unconditional generation.

\noindent\textbf{Baselines.} 
For conditional generation, we compare against three supervised models: T2S~\cite{ijcai2025p580}, Diff-TS~\cite{yuan2024diffusionts}, and TimeVAE~\cite{desai2021timevae}. Additionally, we evaluate two zero-shot foundation models to assess generalization: GPT-4o-mini~\cite{achiam2023gpt} and Llama-3.1-8b~\cite{grattafiori2024llama}. 
For unconditional generation, we select five representative baselines: FreqDiff~\cite{10.5555/3692070.3692444}, Diff-TS, TimeVAE, TimeGAN~\cite{NEURIPS2019_c9efe5f2}, and standard DDPM~\cite{NEURIPS2020_4c5bcfec}.

\noindent\textbf{Datasets.} 
For the conditional task, we employ ETTh1 \& ETTm1~\cite{informer2021}, ECL~\cite{autoformer2021}, Exchange~\cite{10.1145/3209978.3210006}, and Air Quality~\cite{desai2021timevae}. 
For the unconditional task, we synthesize multivariate series on ETTh1, Stocks~\cite{desai2021timevae}, Sines~\cite{yuan2024diffusionts}, and Air Quality (Details in Table~\ref{tab:uncond_datasets}).

\begin{table}[t]
    \centering
    \caption{Description of Unconditional Generation Datasets.}
    \label{tab:uncond_datasets}
    \vspace{-2ex}
    \begin{tabular}{c|cc|c|cc}
        \toprule
        Dataset     & Samples & Dims & Dataset    & Samples & Dims \\
        \midrule
        ETT (h1/m1) & 17420   & 7    &Sines       & 10000   & 5    \\
        Stocks      & 3773    & 6    &Air         & 9333    & 15    \\
        \bottomrule
    \end{tabular}
\vspace{-2ex}
\end{table}

\subsection{Experimental Setting}

\noindent\textbf{Evaluation Metrics.} 
We evaluate unconditional generation using four metrics: 
(1) \textbf{Discriminative Score}~\cite{NEURIPS2019_c9efe5f2}: classification accuracy distinguishing real vs. synthetic data (lower is better); 
(2) \textbf{Predictive Score}~\cite{NEURIPS2019_c9efe5f2}: downstream prediction error via Train-on-Synthetic-Test-on-Real (TSTR); 
(3) \textbf{Context-FID}~\cite{paul2021psa}: Wasserstein distance of latent features; 
and (4) \textbf{Correlational Score}~\cite{10.1145/3490354.3494393}: Frobenius norm of cross-correlation matrix differences.
For conditional tasks, we report \textbf{MSE} and \textbf{WAPE}~\cite{10.1109/TKDE.2024.3484454}, defined as $\sum_{i \in \Omega} | y_i - \hat{y}_i | / \sum_{i \in \Omega} | y_i |$, measuring the deviation from ground truth.

\noindent\textbf{Experimental Details.} 
For the conditional generation, we follow the T2S settings of single-channel fragment-level generation, where LLM-generated textual embeddings modulate the AdaLN coefficients of the DiT backbone across horizons $L \in \{ 24, 48, 96 \}$. 
In the unconditional setting, we sample sequences from Gaussian noise with horizons $L \in \{ 24, 64, 128, 256 \}$, following Diffusion-TS.

Specifically, the interactive correction branch in PaCoDi is integrated within the feedforward block as a cross-quadrature operator. This module processes the concatenated outputs from the parallel attention blocks, formulated as $\text{MLP}(\text{Concat}(\text{Attn.}(\mcR), \text{Attn.}(\mcI)))$, to facilitate interaction between the real and imaginary branches.

\subsection{Performance Analysis}
We evaluate PaCoDi across two primary tasks: \textit{unconditional synthesis} and \textit{conditional generation}. The experimental results demonstrate that PaCoDi achieves the strongest overall ranking / remains the most balanced performer.

\paragraph{\textbf{Conditional Generation.}}
Table~\ref{tab:cond_perf} presents a comparative analysis of PaCoDi against specialized generative models (\eg T2S, Diffusion-TS) and general-purpose LLMs (\eg GPT-4o-mini, Llama 3.1-8b).
The SDE and DDPM variants of PaCoDi exhibit clear dominance, securing the top performance in 19 and 14 evaluation categories, respectively.
This superiority remains consistent across diverse signal dynamics, from highly periodic patterns (\eg ETTh1, ECL) to stochastic volatility (\eg Exchange).
A pivotal advantage of our spectral manifold modeling is its performance stability relative to sequence length; unlike temporal models that suffer from dependency entanglement, PaCoDi scales gracefully.
For instance, on ETTm1 ($L=96$), PaCoDi SDE achieves an MSE of 0.013, outperforming Diffusion-TS (0.031) by a factor of 2.3$\times$.
Notably, while general-purpose LLMs struggle with numerical precision and high-frequency fluctuations, PaCoDi delivers more physically consistent and precise forecasts, demonstrating a native advantage in modeling low-level signal dynamics.

\paragraph{\textbf{Unconditional Generation.}}
Table~\ref{tab:uncond_perf} summarizes the average performance for unconditional synthesis on four datasets across multiple sequence horizons.
PaCoDi, including both its SDE (Cont.) and DDPM (Disc.) variants, exhibits strong overall performance against strong baselines such as Frequency Diffusion, Diffusion-TS, TimeVAE, and TimeGAN.
Collectively, our variants secure 10 first-place and 11 second-place rankings, with the SDE variant alone achieving 9 first-place results.
This consistent performance across diverse datasets underscores the model's robustness in capturing the latent manifold of various temporal dynamics.
A critical finding, validated by our results, is the stability of PaCoDi across the sequence horizon. 
For example, on ETTh1, while Diffusion-TS shows a performance degradation in Discriminative scores at $L=256$, PaCoDi maintains high fidelity. This validates our theoretical claim that parallel quadrature decoupling effectively mitigates the "Curse of Entanglement" in long-horizon signals, allowing the model to resolve complex global dependencies. 

\paragraph{\textbf{Distribution Visualization.}}
Furthermore, Figure~\ref{fig:visualization} visualizes the data distribution on the ETTh1 dataset for unconditional synthesis. As evidenced by the results, the synthetic distribution of PaCoDi DDPM variant aligns more closely with the ground truth than those of Diffusion-TS and standard temporal-domain diffusion.

\begin{table}[t]
    \centering
    \caption{
        Average performance of unconditional generation across horizon $L\in \{ 24, 64, 128, 256 \}$.
        Cont. and Disc. are the symbol of PaCoDi SDE and DDPM variants, respectively.
        \bst{Bold} and \subbst{underline} denote best and second best results.
    }
    \vspace{-2ex}
    \label{tab:uncond_perf}
    \renewcommand{\arraystretch}{0.8}
    \setlength{\tabcolsep}{2.4pt}
    \begin{tabular}{c|c|ccccccc}
        \toprule
        \multicolumn{2}{c|}{\scaleb{Models}} & 
        \textbf{\scaleb{Cont.}} & \textbf{\scaleb{Disc.}} & \scaleb{FreqDiff} & \scaleb{Diff-TS} & 
        \scaleb{TimeVAE} & \scaleb{TimeGAN} & \scaleb{DDPM}    \\
        \midrule

        \multirow{4}{*}{\rotatebox{90}{\scalebox{0.9}{C-FID}}}
        & \scalea{ETTh1} &
        \bst{\scalea{0.202}} & \subbst{\scalea{0.215}} & \scalea{0.259} & \scalea{0.927} & 
        \scalea{0.899} & \scalea{15.771} & \scalea{0.556} \\
        & \scalea{Stocks} &
        \scalea{0.272} & \subbst{\scalea{0.209}} & \bst{\scalea{0.168}} & \scalea{0.436} & 
        \scalea{0.343} & \scalea{2.989} & \scalea{0.636} \\
        & \scalea{Sines} &
        \subbst{\scalea{0.069}} & \scalea{0.100} & \bst{\scalea{0.067}} & \scalea{0.148} & 
        \scalea{0.730} & \scalea{10.766} & \scalea{0.071} \\
        & \scalea{Air} &
        \bst{\scalea{0.339}} & \scalea{0.520} & \scalea{0.601} & \subbst{\scalea{0.397}} & 
        \scalea{1.413} & \scalea{13.452} & \scalea{0.603} \\
        \midrule

        \multirow{4}{*}{\rotatebox{90}{\scalebox{0.9}{Corr.}}}
        & \scalea{ETTh1} &
        \bst{\scalea{0.041}} & \scalea{0.049} & \scalea{0.049} & \scalea{0.086} & 
        \scalea{0.072} & \scalea{1.440} & \subbst{\scalea{0.042}} \\
        & \scalea{Stocks} &
        \scalea{0.011} & \subbst{\scalea{0.008}} & \bst{\scalea{0.007}} & \scalea{0.013} & 
        \scalea{0.087} & \scalea{0.355} & \scalea{0.011} \\
        & \scalea{Sines} &
        \bst{\scalea{0.019}} & \subbst{\scalea{0.021}} & \scalea{0.022} & \scalea{0.030} & 
        \scalea{0.086} & \scalea{0.489} & \subbst{\scalea{0.021}} \\
        & \scalea{Air} &
        \scalea{0.135} & \bst{\scalea{0.128}} & \subbst{\scalea{0.132}} & \scalea{0.195} & 
        \scalea{0.304} & \scalea{0.762} & \scalea{0.234} \\
        \midrule

        \multirow{4}{*}{\rotatebox{90}{\scalebox{0.9}{Discr.}}}
        & \scalea{ETTh1} &
        \bst{\scalea{0.074}} & \subbst{\scalea{0.087}} & \scalea{0.089} & \scalea{0.136} & 
        \scalea{0.176} & \scalea{0.438} & \scalea{0.138} \\
        & \scalea{Stocks} &
        \scalea{0.066} & \subbst{\scalea{0.061}} & \bst{\scalea{0.044}} & \scalea{0.106} & 
        \scalea{0.129} & \scalea{0.342} & \scalea{0.195} \\
        & \scalea{Sines} &
        \bst{\scalea{0.014}} & \subbst{\scalea{0.016}} & \scalea{0.033} & \scalea{0.077} & 
        \scalea{0.068} & \scalea{0.259} & \scalea{0.017} \\
        & \scalea{Air} &
        \scalea{0.205} & \scalea{0.280} & \scalea{0.321} & \bst{\scalea{0.158}} & 
        \scalea{0.421} & \scalea{0.424} & \subbst{\scalea{0.202}} \\
        \midrule

        \multirow{4}{*}{\rotatebox{90}{\scalebox{0.9}{Pred.}}}
        & \scalea{ETTh1} &
        \bst{\scalea{0.107}} & \subbst{\scalea{0.112}} & \scalea{0.119} & \scalea{0.117} & 
        \scalea{0.119} & \scalea{0.161} & \scalea{0.116} \\
        & \scalea{Stocks} &
        \subbst{\scalea{0.037}} & \scalea{0.038} & \bst{\scalea{0.036}} & \subbst{\scalea{0.037}} & 
        \scalea{0.038} & \scalea{0.063} & \scalea{0.042} \\
        & \scalea{Sines} &
        \bst{\scalea{0.205}} & \subbst{\scalea{0.206}} & \bst{\scalea{0.205}} & \scalea{0.208} & 
        \scalea{0.249} & \scalea{0.230} & \subbst{\scalea{0.206}} \\
        & \scalea{Air} &
        \bst{\scalea{0.020}} & \scalea{0.022} & \scalea{0.024} & \subbst{\scalea{0.021}} & 
        \scalea{0.028} & \scalea{0.242} & \scalea{0.022} \\
        \midrule

        \multicolumn{2}{c|}{\scalea{{$1^{\text{st}}$ Count}}} & 
        \bst{\scalea{9}} & \scalea{1} & \subbst{\scalea{6}} & \scalea{1} &
        \scalea{0} & \scalea{0} & \scalea{0} \\
        \midrule
        \multicolumn{2}{c|}{\scalea{{$2^{\text{nd}}$ Count}}} & 
        \scalea{2} & \bst{\scalea{9}} & \scalea{1} & \scalea{3} &
        \scalea{0} & \scalea{0} & \subbst{\scalea{4}} \\
        \bottomrule
    \end{tabular}
\vspace{-2ex}
\end{table}

\begin{table*}[t]
    \centering
    \caption{
        Conditional generation.
        \bst{Bold} and \subbst{underline} denote best and second best results.
    }
    \label{tab:cond_perf}
    \vspace{-2ex}
    \renewcommand{\arraystretch}{0.75}
    \setlength{\tabcolsep}{3.5pt}
    \begin{tabular}{c|c|cc|cc|cc|cc|cc|cc|cc}
        \toprule
        \multicolumn{2}{c}{\scaleb{Models}}         & 
        \multicolumn{2}{c}{\scaleb{PaCoDi SDE}}     & \multicolumn{2}{c}{\scaleb{PaCoDi DDPM}}  &
        \multicolumn{2}{c}{\scaleb{T2S}}            & \multicolumn{2}{c}{\scaleb{Diffusion-TS}} &
        \multicolumn{2}{c}{\scaleb{TimeVAE}}        & \multicolumn{2}{c}{\scaleb{GPT-4o-mini}}  &
        \multicolumn{2}{c}{\scaleb{Llama-3.1-8b}}    \\
        
        \cmidrule(lr){3-4} \cmidrule(lr){5-6} \cmidrule(lr){7-8} \cmidrule(lr){9-10}
        \cmidrule(lr){11-12} \cmidrule(lr){13-14} \cmidrule(lr){15-16}
        
        & \scaleb{$l$} & 
        \scalea{WAPE} & \scalea{MSE} & \scalea{WAPE} & \scalea{MSE} &
        \scalea{WAPE} & \scalea{MSE} & \scalea{WAPE} & \scalea{MSE} &
        \scalea{WAPE} & \scalea{MSE} & \scalea{WAPE} & \scalea{MSE} &
        \scalea{WAPE} & \scalea{MSE} \\
        \toprule
    
        \multirow{4}{*}{\rotatebox{90}{\scalebox{0.9}{ETTh1}}}
        & \scalea{24} & 
        \bst{\scalea{0.152}} & \bst{\scalea{0.005}} & \subbst{\scalea{0.163}} & \subbst{\scalea{0.006}} & 
        \scalea{0.183} & \scalea{0.008} & \scalea{0.793} & \scalea{0.077} & 
        \scalea{0.666} & \scalea{0.055} & \scalea{0.264} & \scalea{0.041} & 
        \scalea{0.883} & \scalea{0.663} \\
        & \scalea{48} & 
        \bst{\scalea{0.163}} & \bst{\scalea{0.005}} & \subbst{\scalea{0.166}} & \subbst{\scalea{0.006}} & 
        \scalea{0.234} & \scalea{0.013} & \scalea{1.207} & \scalea{0.120} & 
        \scalea{0.647} & \scalea{0.055} & \scalea{0.414} & \scalea{0.080} & 
        \scalea{0.923} & \scalea{1.260} \\
        & \scalea{96} & 
        \bst{\scalea{0.154}} & \bst{\scalea{0.005}} & \subbst{\scalea{0.155}} & \subbst{\scalea{0.006}} & 
        \scalea{0.229} & \scalea{0.011} & \scalea{0.498} & \scalea{0.028} & 
        \scalea{0.643} & \scalea{0.055} & \scalea{0.500} & \scalea{0.118} & 
        \scalea{0.949} & \scalea{1.748} \\
        \cmidrule(lr){2-16}
        & \scalea{\emph{Avg}} & 
        \bst{\scalea{0.156}} & \bst{\scalea{0.005}} & \subbst{\scalea{0.161}} & \subbst{\scalea{0.006}} & 
        \scalea{0.215} & \scalea{0.011} & \scalea{0.833} & \scalea{0.075} & 
        \scalea{0.652} & \scalea{0.055} & \scalea{0.393} & \scalea{0.080} & 
        \scalea{0.918} & \scalea{1.224} \\
        \midrule
    
        \multirow{4}{*}{\rotatebox{90}{\scalebox{0.9}{ETTm1}}}
        & \scalea{24} & 
        \bst{\scalea{0.169}} & \bst{\scalea{0.008}} & \subbst{\scalea{0.170}} & \bst{\scalea{0.008}} & 
        \scalea{0.426} & \scalea{0.033} & \scalea{0.604} & \scalea{0.040} & 
        \scalea{0.666} & \scalea{0.048} & \scalea{0.244} & \subbst{\scalea{0.031}} & 
        \scalea{1.134} & \scalea{0.798} \\
        & \scalea{48} & 
        \subbst{\scalea{0.196}} & \subbst{\scalea{0.010}} & \bst{\scalea{0.192}} & \bst{\scalea{0.009}} & 
        \scalea{0.530} & \scalea{0.053} & \scalea{1.119} & \scalea{0.100} & 
        \scalea{0.636} & \scalea{0.051} & \scalea{0.453} & \scalea{0.112} & 
        \scalea{1.074} & \scalea{1.496} \\
        & \scalea{96} & 
        \bst{\scalea{0.228}} & \bst{\scalea{0.013}} & \subbst{\scalea{0.230}} & \bst{\scalea{0.013}} & 
        \scalea{0.414} & \scalea{0.041} & \scalea{0.546} & \subbst{\scalea{0.031}} & 
        \scalea{0.664} & \scalea{0.057} & \scalea{0.706} & \scalea{0.395} & 
        \scalea{1.079} & \scalea{1.761} \\
        \cmidrule(lr){2-16}
        & \scalea{\emph{Avg}} & 
        \subbst{\scalea{0.198}} & \bst{\scalea{0.010}} & \bst{\scalea{0.197}} & \bst{\scalea{0.010}} & 
        \scalea{0.457} & \subbst{\scalea{0.042}} & \scalea{0.756} & \scalea{0.057} & 
        \scalea{0.655} & \scalea{0.052} & \scalea{0.468} & \scalea{0.179} & 
        \scalea{1.096} & \scalea{1.352} \\
        \midrule
    
        \multirow{4}{*}{\rotatebox{90}{\scalebox{0.9}{ECL}}}
        & \scalea{24} & 
        \bst{\scalea{0.104}} & \bst{\scalea{0.006}} & \subbst{\scalea{0.109}} & \bst{\scalea{0.006}} & 
        \scalea{0.135} & \subbst{\scalea{0.010}} & \scalea{0.617} & \scalea{0.041} & 
        \scalea{0.207} & \scalea{0.016} & \scalea{0.734} & \scalea{0.592} & 
        \scalea{0.926} & \scalea{1.140} \\
        & \scalea{48} & 
        \bst{\scalea{0.135}} & \bst{\scalea{0.010}} & \bst{\scalea{0.135}} & \bst{\scalea{0.010}} & 
        \subbst{\scalea{0.155}} & \subbst{\scalea{0.013}} & \scalea{1.128} & \scalea{0.102} & 
        \scalea{0.208} & \scalea{0.017} & \scalea{1.014} & \scalea{1.065} & 
        \scalea{1.038} & \scalea{1.416} \\
        & \scalea{96} & 
        \subbst{\scalea{0.155}} & \subbst{\scalea{0.013}} & \bst{\scalea{0.153}} & \bst{\scalea{0.012}} & 
        \scalea{0.238} & \scalea{0.031} & \scalea{0.545} & \scalea{0.032} & 
        \scalea{0.213} & \scalea{0.018} & \scalea{1.024} & \scalea{1.210} & 
        \scalea{1.085} & \scalea{1.740} \\
        \cmidrule(lr){2-16}
        & \scalea{\emph{Avg}} & 
        \bst{\scalea{0.131}} & \subbst{\scalea{0.010}} & \subbst{\scalea{0.132}} & \bst{\scalea{0.009}} & 
        \scalea{0.176} & \scalea{0.018} & \scalea{0.763} & \scalea{0.058} & 
        \scalea{0.209} & \scalea{0.017} & \scalea{0.924} & \scalea{0.956} & 
        \scalea{1.016} & \scalea{1.432} \\
        \midrule
    
        \multirow{4}{*}{\rotatebox{90}{\scalebox{0.9}{Exchange}}}
        & \scalea{24} & 
        \bst{\scalea{0.265}} & \bst{\scalea{0.032}} & \subbst{\scalea{0.266}} & \bst{\scalea{0.032}} & 
        \scalea{0.292} & \subbst{\scalea{0.033}} & \scalea{0.791} & \scalea{0.077} & 
        \scalea{1.165} & \scalea{0.105} & \scalea{1.072} & \scalea{2.060} & 
        \scalea{1.258} & \scalea{2.052} \\
        & \scalea{48} & 
        \bst{\scalea{0.254}} & \subbst{\scalea{0.030}} & \bst{\scalea{0.254}} & \bst{\scalea{0.029}} & 
        \subbst{\scalea{0.259}} & \scalea{0.033} & \scalea{1.217} & \scalea{0.122} & 
        \scalea{1.064} & \scalea{0.106} & \scalea{0.933} & \scalea{1.074} & 
        \scalea{1.562} & \scalea{2.125} \\
        & \scalea{96} & 
        \subbst{\scalea{0.237}} & \bst{\scalea{0.025}} & \bst{\scalea{0.233}} & \bst{\scalea{0.025}} & 
        \scalea{0.480} & \subbst{\scalea{0.047}} & \scalea{0.504} & \scalea{0.048} & 
        \scalea{0.977} & \scalea{0.106} & \scalea{1.141} & \scalea{1.625} & 
        \scalea{1.433} & \scalea{1.892} \\
        \cmidrule(lr){2-16}
        & \scalea{\emph{Avg}} & 
        \subbst{\scalea{0.252}} & \bst{\scalea{0.029}} & \bst{\scalea{0.251}} & \bst{\scalea{0.029}} & 
        \scalea{0.344} & \subbst{\scalea{0.038}} & \scalea{0.837} & \scalea{0.082} & 
        \scalea{1.069} & \scalea{0.106} & \scalea{1.049} & \scalea{1.586} & 
        \scalea{1.418} & \scalea{2.023} \\
        \midrule
    
        \multirow{4}{*}{\rotatebox{90}{\scalebox{0.9}{Air Quality}}}
        & \scalea{24} & 
        \subbst{\scalea{0.662}} & \bst{\scalea{0.013}} & \scalea{0.699} & \subbst{\scalea{0.014}} & 
        \scalea{0.884} & \scalea{0.020} & \scalea{0.806} & \scalea{0.078} & 
        \scalea{2.303} & \scalea{0.022} & \bst{\scalea{0.557}} & \scalea{0.379} & 
        \scalea{0.878} & \scalea{0.697} \\
        & \scalea{48} & 
        \scalea{0.961} & \scalea{0.029} & \subbst{\scalea{0.824}} & \subbst{\scalea{0.028}} & 
        \scalea{1.295} & \scalea{0.044} & \scalea{1.439} & \scalea{0.120} & 
        \scalea{1.648} & \bst{\scalea{0.023}} & \bst{\scalea{0.791}} & \scalea{0.715} & 
        \scalea{1.141} & \scalea{1.642} \\
        & \scalea{96} & 
        \subbst{\scalea{0.858}} & \subbst{\scalea{0.026}} & \scalea{0.915} & \scalea{0.030} & 
        \scalea{1.377} & \scalea{0.049} & \bst{\scalea{0.508}} & \scalea{0.028} & 
        \scalea{1.270} & \bst{\scalea{0.024}} & \scalea{0.928} & \scalea{1.127} & 
        \scalea{1.085} & \scalea{1.551} \\
        \cmidrule(lr){2-16}
        & \scalea{\emph{Avg}} & 
        \scalea{0.827} & \bst{\scalea{0.023}} & \subbst{\scalea{0.813}} & \subbst{\scalea{0.024}} & 
        \scalea{1.185} & \scalea{0.038} & \scalea{0.918} & \scalea{0.075} & 
        \scalea{1.740} & \bst{\scalea{0.023}} & \bst{\scalea{0.759}} & \scalea{0.740} & 
        \scalea{1.035} & \scalea{1.297} \\
        \midrule
        
        \multicolumn{2}{c|}{\scalea{{$1^{\text{st}}$ Count}}} & 
        \bst{\scalea{9}} & \bst{\scalea{10}} & \subbst{\scalea{5}} & \subbst{\scalea{9}} & 
        \scalea{0} & \scalea{0} & \scalea{1} & \scalea{0} & 
        \scalea{0} & \scalea{2} & \scalea{2} & \scalea{0} & 
        \scalea{0} & \scalea{0} \\
        \midrule
        \multicolumn{2}{c|}{\scalea{{$2^{\text{nd}}$ Count}}} & 
        \subbst{\scalea{5}} & \subbst{\scalea{4}} & \bst{\scalea{8}} & \bst{\scalea{5}} & 
        \scalea{2} & \subbst{\scalea{4}} & \scalea{0} & \scalea{1} & 
        \scalea{0} & \scalea{0} & \scalea{0} & \scalea{1} & 
        \scalea{0} & \scalea{0} \\
        \bottomrule
    \end{tabular}
\vspace{-2ex}
\end{table*}

\begin{figure}
    \centering
    \includegraphics[width=0.95\linewidth]{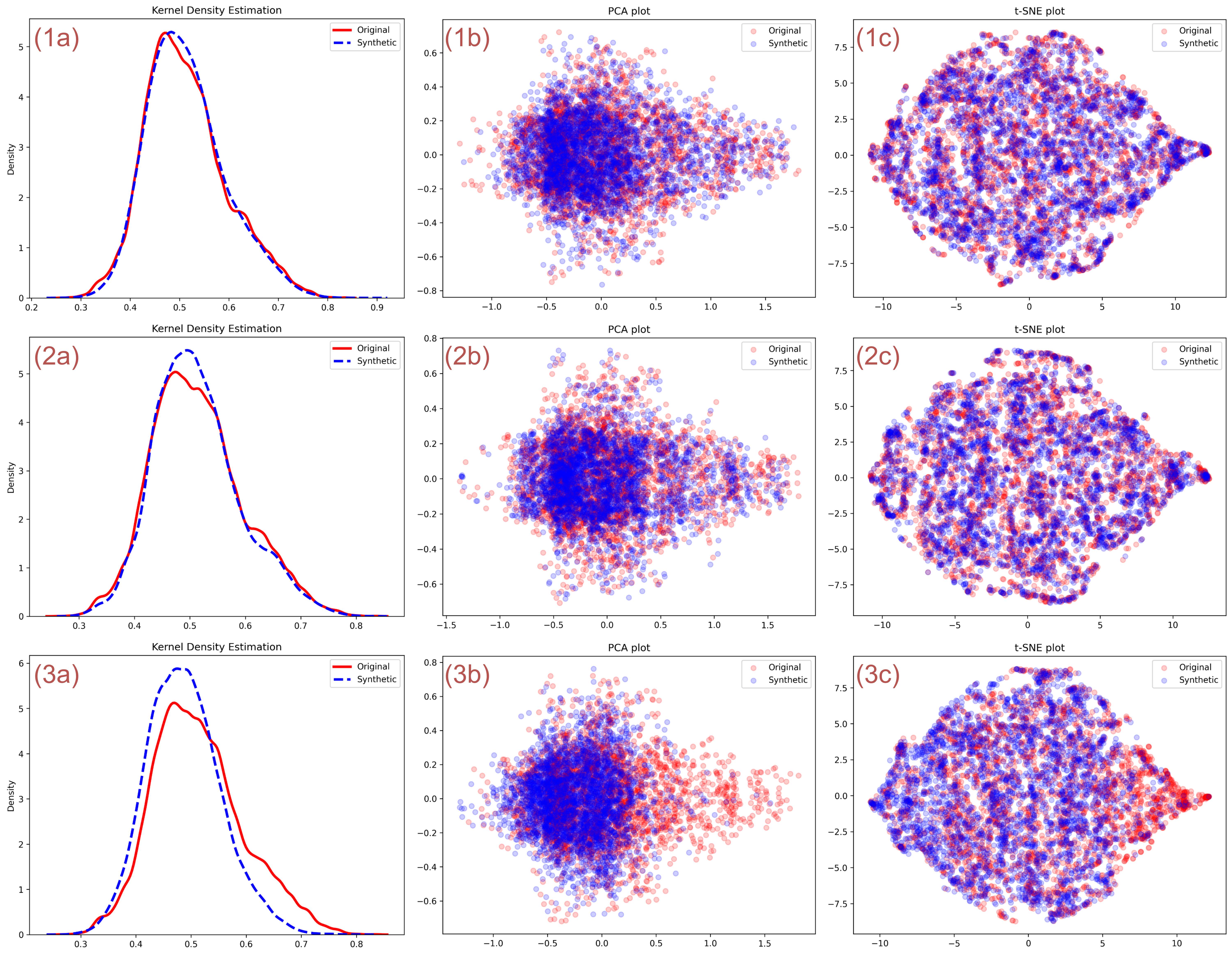}
    \vspace{-2ex}
    \caption{
        Data Distribution Visualization on ETTh1 ($L = 64$).
        Rows (1)–(3) correspond to PaCoDi DDPM, Diffusion-TS, and vanilla Diffusion. 
        Columns (a)–(c) display Data Density Estimation, PCA, and t-SNE projections. 
        Red and blue markers denote real and synthetic samples, respectively.
    }
    \label{fig:visualization}
    \Description{
        Architecture evolution from temporal to parallel complex diffusion.
    }
\end{figure}

\subsection{Computational Complexity Analysis}

\begin{figure}
    \centering
    \includegraphics[width=\linewidth]{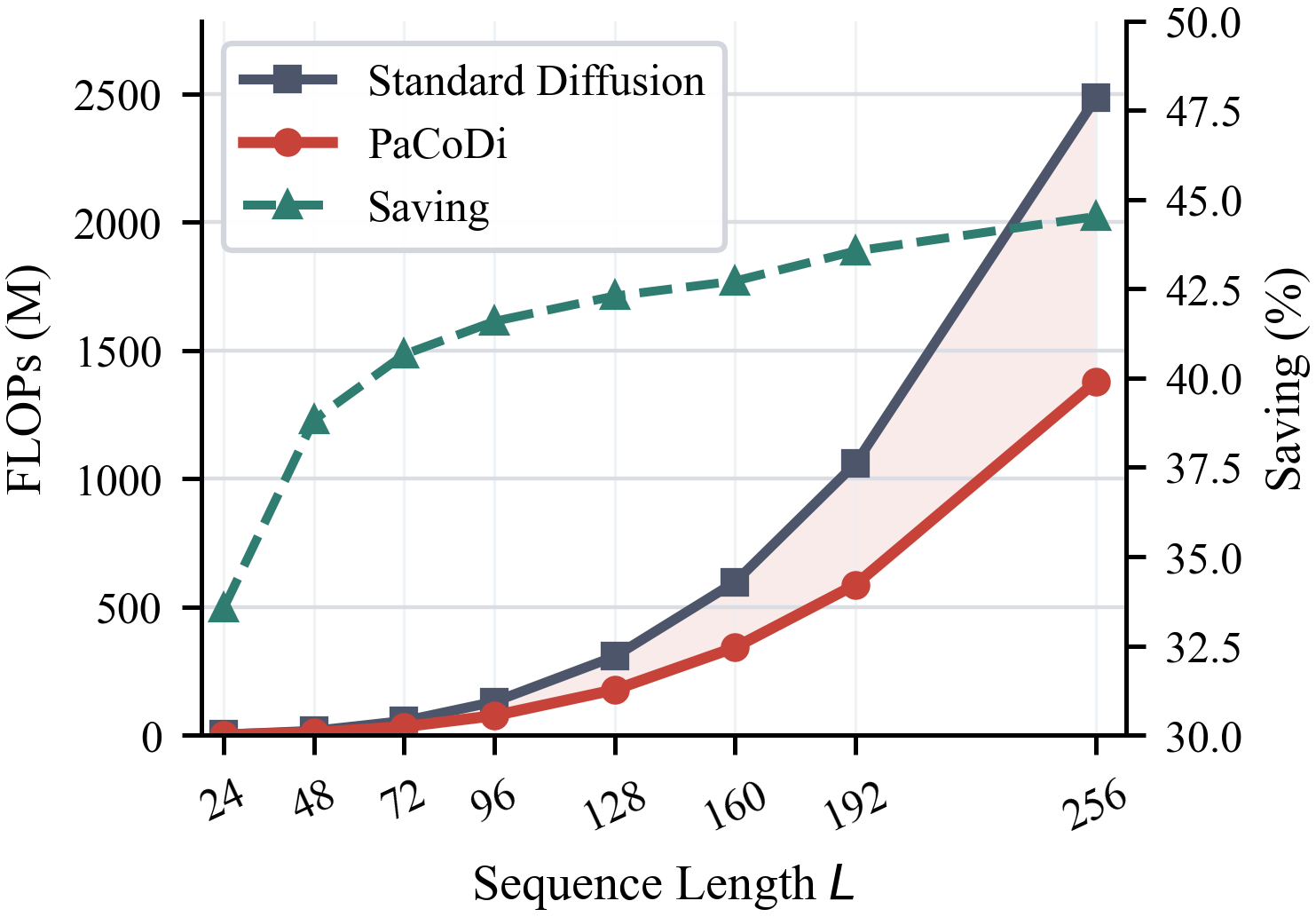}
    \caption{Computational Complexity Analysis.}
    \label{fig:complexity_analysis}
    \Description{Computational Complexity Analysis.}
\end{figure}

A primary advantage of PaCoDi is its significantly reduced computational overhead, achieved through the parallel modeling of quadrature components. To quantify this efficiency gain, we evaluate the computational complexity across sequence horizons ranging from $L=24$ to 256. 
Specifically, to maintain consistent model capacity, the hidden dimension $h$ is scaled proportionally with the sequence horizon $L$.
In this regime, the self-attention mechanism becomes the primary driver of the total computational cost as $L$ increases.
As illustrated in Figure~\ref{fig:complexity_analysis}, the computational cost for both PaCoDi and standard temporal diffusion grows quadratically ($\mcO(L^2)$) relative to $L$. 
Notably, PaCoDi maintains a substantially lower cost profile; the FLOP savings ratio increases rapidly with $L$ and asymptotically converges to 50\%. 
This aligns with our theoretical expectation that bifurcating the attention block effectively bisects the quadratic scaling bottleneck of monolithic DiT.

\subsection{Ablation Validation}
We evaluate the impact of cross-quadrature interaction by comparing PaCoDi against a strictly decoupled spectral variant (Dec.) and a monolithic temporal baseline (Temp.) across continuous (SDE) and discrete (DDPM) frameworks.

\textbf{Parallel vs. Monolithic Architecture.}
Table~\ref{tab:uncond_perf} also provides a direct comparison between PaCoDi's parallel quadrature architecture and Frequency Diffusion's real-image concatenated monolithic modeling.
PaCoDi maintains strong overall generative quality across most datasets, while Frequency Diffusion is particularly competitive on Stocks, where strong phase coherence makes the real-imaginary interaction harder to approximate with a lightweight parallel structure.
This exposes the main trade-off between the two spectral designs: Frequency Diffusion preserves a simple monolithic estimator maintain phase coherence but without mathematical reasonableness, whereas PaCoDi decomposes the complex dynamics into parallel real-valued branches with an interactive correction branch.
From the efficiency perspective, Figure~\ref{fig:complexity_analysis} shows that PaCoDi has a consistently lower computational cost than standard diffusion.
In contrast, PaCoDi exploits Hermitian compression and parallel quadrature branches, reducing the attention complexity to one half while preserving competitive performance.

\textbf{The Necessity of Interactive Coupling.} 
As evidenced in Table~\ref{tab:ablation}, while the Dec. benefits from parallel efficiency, it suffers from a severe collapse in generative fidelity, characterized by significant degradation in Context-FID and Discriminative scores. 
This performance gap empirically validates our theoretical hypothesis: treating quadrature components as entirely independent leads to the marginalization of essential phase-amplitude correlations. 
Notably, by incorporating the interactive correction branch, PaCoDi successfully recovers these vital dependencies, achieving superior fidelity that even surpasses the temporal baseline.

\textbf{Robustness Across Frequency Space.} 
The performance gains are consistent across both discrete and continuous settings. PaCoDi secures the top rankings, demonstrating that our MFT compensation mechanism effectively maintains the coherence of real and imaginary flows even under unknown data distributions, bridging the gap between computational scalability and generative integrity.

\begin{table}[h]
    \centering
    \caption{
        \textbf{Ablation Experiments on Sines dataset}.
        \bst{Bold} and \subbst{underline} denote best and second best results.
    }
    \label{tab:ablation}
    \vspace{-2ex}
    \renewcommand{\arraystretch}{0.8}
    \setlength{\tabcolsep}{3.5pt}
    \begin{tabular}{c|c|ccc|ccc}
        \toprule
        \multicolumn{2}{c|}{\multirow{2}{*}{\scaleb{Models}}} & 
        \multicolumn{3}{c|}{\scalea{DDPM}} & \multicolumn{3}{c}{\scalea{SDE}} \\

        \cmidrule(lr){3-5} \cmidrule(lr){6-8}
        
        \multicolumn{2}{c|}{} & 
        \scaleb{PaCoDi} & \scaleb{Dec.} & \scaleb{Temp.} & 
        \scaleb{PaCoDi} & \scaleb{Dec.} & \scaleb{Temp.}  \\
        \midrule
        
        \multirow{5}{*}{\rotatebox{90}{\scalebox{0.9}{Context-FID}}} 
        & \scalea{24} &
        \bst{\scalea{$0.003_{\pm .001}$}} & \scalea{$0.264_{\pm .055}$} & \subbst{\scalea{$0.038_{\pm .010}$}} & 
        \bst{\scalea{$0.007_{\pm .001}$}} & \scalea{$0.283_{\pm .052}$} & \subbst{\scalea{$0.228_{\pm .036}$}} \\
        & \scalea{64} &
        \bst{\scalea{$0.018_{\pm .003}$}} & \scalea{$1.573_{\pm .177}$} & \subbst{\scalea{$0.070_{\pm .023}$}} & 
        \bst{\scalea{$0.012_{\pm .002}$}} & \scalea{$1.621_{\pm .277}$} & \subbst{\scalea{$0.141_{\pm .010}$}} \\
        & \scalea{128} &
        \bst{\scalea{$0.062_{\pm .005}$}} & \scalea{$0.823_{\pm .074}$} & \subbst{\scalea{$0.163_{\pm .046}$}} & 
        \bst{\scalea{$0.043_{\pm .002}$}} & \scalea{$0.874_{\pm .186}$} & \subbst{\scalea{$0.162_{\pm .008}$}} \\
        & \scalea{256} &
        \subbst{\scalea{$0.317_{\pm .014}$}} & \scalea{$0.883_{\pm .134}$} & \bst{\scalea{$0.311_{\pm .321}$}} & 
        \subbst{\scalea{$0.213_{\pm .022}$}} & \scalea{$1.081_{\pm .083}$} & \bst{\scalea{$0.128_{\pm .011}$}} \\
        \cmidrule(lr){2-8}
        & \scalea{\emph{Avg}} &
        \bst{\scalea{0.100}} & \scalea{0.886} & \subbst{\scalea{0.146}} & 
        \bst{\scalea{0.069}} & \scalea{0.965} & \subbst{\scalea{0.165}} \\
        \midrule

        \multirow{5}{*}{\rotatebox{90}{\scalebox{0.9}{Correlational}}}
        & \scalea{24} & 
        \subbst{\scalea{$0.019_{\pm .006}$}} & \bst{\scalea{$0.018_{\pm .007}$}} & \scalea{$0.030_{\pm .004}$} & 
        \subbst{\scalea{$0.026_{\pm .011}$}} & \bst{\scalea{$0.018_{\pm .006}$}} & \scalea{$0.082_{\pm .014}$} \\
        & \scalea{64} & 
        \scalea{$0.030_{\pm .010}$} & \subbst{\scalea{$0.020_{\pm .007}$}} & \bst{\scalea{$0.019_{\pm .004}$}} & 
        \bst{\scalea{$0.023_{\pm .007}$}} & \subbst{\scalea{$0.025_{\pm .005}$}} & \scalea{$0.030_{\pm .012}$} \\
        & \scalea{128} & 
        \scalea{$0.021_{\pm .003}$} & \subbst{\scalea{$0.015_{\pm .005}$}} & \bst{\scalea{$0.012_{\pm .007}$}} & 
        \subbst{\scalea{$0.017_{\pm .004}$}} & \scalea{$0.020_{\pm .003}$} & \bst{\scalea{$0.016_{\pm .003}$}} \\
        & \scalea{256} & 
        \bst{\scalea{$0.015_{\pm .004}$}} & \bst{\scalea{$0.015_{\pm .005}$}} & \subbst{\scalea{$0.032_{\pm .008}$}} & 
        \bst{\scalea{$0.011_{\pm .005}$}} & \subbst{\scalea{$0.015_{\pm .007}$}} & \scalea{$0.017_{\pm .003}$} \\
        \cmidrule(lr){2-8}
        & \scalea{\emph{Avg}} & 
        \subbst{\scalea{0.021}} & \bst{\scalea{0.017}} & \scalea{0.023} & 
        \bst{\scalea{0.019}} & \subbst{\scalea{0.020}} & \scalea{0.036} \\
        \midrule

        \multirow{5}{*}{\rotatebox{90}{\scalebox{0.9}{Discriminative}}}
        & \scalea{24} & 
        \bst{\scalea{$0.007_{\pm .008}$}} & \scalea{$0.052_{\pm .049}$} & \subbst{\scalea{$0.019_{\pm .008}$}} & 
        \bst{\scalea{$0.008_{\pm .004}$}} & \subbst{\scalea{$0.035_{\pm .037}$}} & \scalea{$0.136_{\pm .198}$} \\
        & \scalea{64} & 
        \bst{\scalea{$0.010_{\pm .003}$}} & \scalea{$0.136_{\pm .038}$} & \subbst{\scalea{$0.027_{\pm .041}$}} & 
        \bst{\scalea{$0.008_{\pm .005}$}} & \scalea{$0.156_{\pm .083}$} & \subbst{\scalea{$0.009_{\pm .007}$}} \\
        & \scalea{128} & 
        \subbst{\scalea{$0.022_{\pm .007}$}} & \scalea{$0.213_{\pm .051}$} & \bst{\scalea{$0.011_{\pm .013}$}} & 
        \bst{\scalea{$0.013_{\pm .016}$}} & \scalea{$0.196_{\pm .106}$} & \subbst{\scalea{$0.041_{\pm .015}$}} \\
        & \scalea{256} & 
        \bst{\scalea{$0.023_{\pm .019}$}} & \scalea{$0.183_{\pm .070}$} & \subbst{\scalea{$0.066_{\pm .142}$}} & 
        \bst{\scalea{$0.027_{\pm .006}$}} & \scalea{$0.194_{\pm .096}$} & \subbst{\scalea{$0.043_{\pm .029}$}} \\
        \cmidrule(lr){2-8}
        & \scalea{\emph{Avg}} & 
        \bst{\scalea{0.016}} & \scalea{0.146} & \subbst{\scalea{0.031}} & 
        \bst{\scalea{0.014}} & \scalea{0.145} & \subbst{\scalea{0.057}} \\
        \midrule

        \multirow{5}{*}{\rotatebox{90}{\scalebox{0.9}{Predictive}}}
        & \scalea{24} & 
        \bst{\scalea{$0.093_{\pm .000}$}} & \subbst{\scalea{$0.095_{\pm .000}$}} & \scalea{$0.188_{\pm .007}$} & 
        \bst{\scalea{$0.093_{\pm .000}$}} & \subbst{\scalea{$0.094_{\pm .000}$}} & \bst{\scalea{$0.093_{\pm .000}$}} \\
        & \scalea{64} & 
        \bst{\scalea{$0.191_{\pm .002}$}} & \subbst{\scalea{$0.192_{\pm .003}$}} & \scalea{$0.255_{\pm .001}$} & 
        \bst{\scalea{$0.188_{\pm .003}$}} & \scalea{$0.195_{\pm .003}$} & \subbst{\scalea{$0.191_{\pm .002}$}} \\
        & \scalea{128} & 
        \bst{\scalea{$0.252_{\pm .003}$}} & \subbst{\scalea{$0.260_{\pm .005}$}} & \scalea{$0.288_{\pm .003}$} & 
        \bst{\scalea{$0.252_{\pm .003}$}} & \scalea{$0.260_{\pm .006}$} & \subbst{\scalea{$0.253_{\pm .004}$}} \\
        & \scalea{256} & 
        \subbst{\scalea{$0.287_{\pm .003}$}} & \scalea{$0.291_{\pm .004}$} & \bst{\scalea{$0.119_{\pm .001}$}} & 
        \bst{\scalea{$0.286_{\pm .003}$}} & \scalea{$0.290_{\pm .001}$} & \subbst{\scalea{$0.287_{\pm .003}$}} \\
        \cmidrule(lr){2-8}
        & \scalea{\emph{Avg}} & 
        \bst{\scalea{0.206}} & \subbst{\scalea{0.210}} & \scalea{0.213} & 
        \bst{\scalea{0.205}} & \scalea{0.210} & \subbst{\scalea{0.206}} \\
        \bottomrule

    \end{tabular}
% \vspace{-10pt}
\end{table}

\section{Conclusion}
This work revisits time series diffusion from the perspective of representation geometry. 
Instead of treating temporal dependence as a burden to be absorbed entirely by a stronger denoising estimator, PaCoDi moves the diffusion process itself into the spectral domain, where the DFT exposes a less entangled modal structure. 
In this sense, the frequency domain is no longer an auxiliary feature space for neural blocks, but the native state space in which the stochastic generative path is defined.
The main outcome is a principled way to make this shift mathematically valid: complex spectral dynamics are realized through parallel real-valued processes, with residual coupling handled by correction rather than by an unconstrained monolithic estimator.
This formulation also turns the Hermitian structure of real signals into a concrete efficiency advantage, reducing 50\% attention complexity while preserving the information content of the sequence.
Together with the empirical results on unconditional and conditional generation, PaCoDi suggests that future time series diffusion models should co-design the diffusion path with the geometry of the data distribution, rather than only scaling the estimator under a fixed temporal path.

\section{Acknowledgments}
This work is supported by the "Pioneer" and "Leading Goose" R\&D Program of Zhejiang (No. 2026C02A1250) and by the National Natural Science Foundation of China (No. 62503419).

\bibliographystyle{ACM-Reference-Format}
\balance
\bibliography{refs}

\newpage
\appendix
\onecolumn

% \paragraph{Appendix roadmap.}
% The appendices proceed from spectral noise statistics to reverse-process theory and implementation-facing objectives:
% \begin{itemize}
%     \item \textbf{Appendix~\ref{app:spec_dist_gauss}} derives the Hermitian spectral Gaussian law induced by real temporal noise.
%     \item \textbf{Appendix~\ref{app:cond_reverse_factor}} proves the conditional reverse factorization that supports parallel real and imaginary branches.
%     \item \textbf{Appendix~\ref{app:hetero_reverse_derivation}} derives DDPM reverse dynamics under general heteroscedastic Gaussian noise and justifies the Mahalanobis objective.
%     \item \textbf{Appendix D} develops the Frequency SDE framework, including reverse dynamics, score--noise identities, mean-field marginalization, and loss equivalence.
%     \item \textbf{Appendices~\ref{app:math_lemmas}--\ref{app:results}} collect supporting lemmas and supplementary experimental results.
% \end{itemize}

\section{Statistical Property of Hermitian Complex Gaussian Noise}
\label{app:spec_dist_gauss}

In this section, we derive the statistical properties of the spectral projection of i.i.d. temporal Gaussian noise, \ie \textit{Hermitian Complex Gaussian Noise}.
Let $\beps \in \R^{L}$ be a temporal Gaussian noise vector with $\beps \sim \mcN(\mbzero, \sigma^2 \mbI)$, and let
\begin{equation}
    \mcE = \mcF(\beps) = \bvareps_r + j\bvareps_i
\end{equation}
be its normalized DFT projection.
Because $\beps$ is real-valued, $\mcE$ is not an unconstrained proper complex Gaussian over the full space $\C^L$.
Instead, it is a Gaussian random vector supported on the real-linear Hermitian spectral subspace of $\C^L$, satisfying $\mcE_{(-k)\bmod L}=\overline{\mcE_k}$.
We show that this DFT-induced Hermitian Gaussian noise satisfies:
\begin{itemize}
    \item \textbf{Quadrature Independence:} The real and imaginary components $\bvareps_r$ and $\bvareps_i$ are statistically independent as random vectors.
    \item \textbf{Spectral Heteroscedasticity:} Their marginal distributions are Gaussian with covariance matrices $\bSigma_r$ and $\bSigma_i$, whose diagonal and anti-diagonal structures are induced by Hermitian symmetry.
\end{itemize}

\subsection{Setup and Preliminaries}
Let $n \in [0, L-1]$ denote the temporal index and $k \in [0, L-1]$ denote the frequency-bin index.
The Discrete Fourier Transform (DFT) of the noise vector $\beps$ is defined as:
\begin{equation}
    \mcE_{k}
    = \frac{1}{\sqrt{L}} \sum_{n=0}^{L-1} \beps_{n} \exp\left(-j \frac{2\pi k n}{L}\right)
    = \bvareps_{r,k} + j \bvareps_{i,k},
\end{equation}
where the real and imaginary components are:
\begin{equation}
    \bvareps_{r,k} 
    = \frac{1}{\sqrt{L}} \sum_{n=0}^{L-1} \beps_{n} \cos\left(\frac{2\pi k n}{L}\right) , \quad
    \bvareps_{i,k} 
    = - \frac{1}{\sqrt{L}} \sum_{n=0}^{L-1} \beps_{n} \sin\left(\frac{2\pi k n}{L}\right) .
\end{equation}

Since the DFT is linear and $\beps$ is Gaussian, the augmented vector $[\bvareps_r^\top,\bvareps_i^\top]^\top$ is jointly Gaussian.
We use the modular indicator
\begin{equation}
    \delta_L(A)=
    \begin{cases}
        1 & \text{if } A \equiv 0 \pmod L,\\
        0 & \text{otherwise},
    \end{cases}
\end{equation}
and the following trigonometric identity (Lemma~\ref{lem:trig_sum}):
\begin{equation}
    \sum_{n=0}^{L-1} \cos\left(\frac{2\pi A n}{L}\right) =
    L \cdot \delta_L(A)
    \label{eq:trig_identity}
\end{equation}
where $A$ will be instantiated as $k-l$ or $k+l$.
The sine counterpart satisfies $\sum_{n=0}^{L-1}\sin(2\pi A n/L)=0$ for every integer $A$.

\subsection{Quadrature Orthogonality}

The first moments vanish by linearity:
\begin{equation}
    \expt[\bvareps_{r,k}] = 0, \quad
    \expt[\bvareps_{i,k}] = 0,
    \quad \forall k.
\end{equation}

For the cross-covariance between real frequency $k$ and imaginary frequency $l$, Lemma~\ref{lem:exp_prod_Gauss_noise} gives:
\begin{equation}
\begin{aligned}
    \cov(\bvareps_{r,k}, \bvareps_{i,l})
    &= \expt \left[
        \left(  \frac{1}{\sqrt{L}} \sum_{n=0}^{L-1} \beps_{n} \cos\frac{2\pi kn}{L} \right)
        \left(- \frac{1}{\sqrt{L}} \sum_{m=0}^{L-1} \beps_{m} \sin\frac{2\pi lm}{L} \right)
    \right]
    \\
    &= - \frac{1}{L} \sum_{n=0}^{L-1} \sum_{m=0}^{L-1} 
        \underbrace{\expt \left[ \beps_{n} \beps_{m} \right]}_{\sigma^2 \delta_{n, m}} 
        \cos \frac{2\pi kn}{L} \sin\frac{2\pi lm}{L}
    \\
    &= - \frac{\sigma^2}{L} \sum_{n=0}^{L-1} \cos\frac{2\pi kn}{L} \sin\frac{2\pi ln}{L}
    \\
    &= - \frac{\sigma^2}{2L} \sum_{n=0}^{L-1} \left( \sin\frac{2\pi n(k+l)}{L} - \sin\frac{2\pi n(k-l)}{L} \right).
\end{aligned}
\end{equation}

Both sine sums vanish, hence
\begin{equation}
\cov(\bvareps_{r,k}, \bvareps_{i,l}) = 0, \quad \forall k, l.
\end{equation}

Since $[\bvareps_r^\top,\bvareps_i^\top]^\top$ is jointly Gaussian, zero cross-covariance implies statistical independence between the two quadrature vectors:
\begin{equation}
    \bvareps_r \perp \bvareps_i .
\end{equation}

\subsection{Hermitian Heteroscedastic Covariance}

We next derive the covariance within each quadrature component.
For the real part:
\begin{equation}
\begin{aligned}
    \cov(\bvareps_{r,k}, \bvareps_{r,l})
    &= \expt \left[ 
        \left( \frac{1}{\sqrt{L}} \sum_{n=0}^{L-1} \beps_{n} \cos \frac{2\pi kn}{L} \right) 
        \left( \frac{1}{\sqrt{L}} \sum_{m=0}^{L-1} \beps_{m} \cos \frac{2\pi lm}{L} \right) 
    \right] 
    \\
    &= \frac{1}{L} \sum_{n=0}^{L-1} \sum_{m=0}^{L-1} 
        \underbrace{\expt \left[ \beps_{n} \beps_{m} \right]}_{\sigma^2 \delta_{n,m}}
        \cos \frac{2\pi kn}{L}  \cos \frac{2\pi lm}{L}
    \\
    &= \frac{\sigma^2}{L} \sum_{n=0}^{L-1} \cos\frac{2\pi kn}{L} \cos\frac{2\pi ln}{L}
    \\
    &= \frac{\sigma^2}{2L} \sum_{n=0}^{L-1} \left(
        \cos\frac{2\pi n(k-l)}{L} + \cos\frac{2\pi n(k+l)}{L}
    \right).
\end{aligned}
\end{equation}

Using Eq.~\eqref{eq:trig_identity}, we obtain the compact identity:
\begin{equation}
    \cov(\bvareps_{r,k}, \bvareps_{r,l}) 
    = \frac{\sigma^2}{2}\left[\delta_L(k-l)+\delta_L(k+l)\right].
\label{eq:cov_R_R}
\end{equation}

Similarly, for the imaginary part:
\begin{equation} 
\begin{aligned} 
    \cov(\bvareps_{i,k}, \bvareps_{i,l}) 
    &= \expt \left[ 
        \left(- \frac{1}{\sqrt{L}} \sum_{n=0}^{L-1} \beps_{n} \sin \frac{2\pi kn}{L} \right) 
        \left(- \frac{1}{\sqrt{L}} \sum_{m=0}^{L-1} \beps_{m} \sin \frac{2\pi lm}{L} \right) 
    \right] 
    \\
    &= \frac{1}{L} \sum_{n=0}^{L-1} \sum_{m=0}^{L-1} 
        \expt \left[ \beps_{n} \beps_{m} \right] 
        \sin \frac{2\pi kn}{L}  \sin \frac{2\pi lm}{L}
    \\
    &= \frac{\sigma^2}{L} \sum_{n=0}^{L-1} \sin\frac{2\pi kn}{L} \sin\frac{2\pi ln}{L} 
    \\ 
    &= \frac{\sigma^2}{2L} \sum_{n=0}^{L-1} \left( \cos\frac{2\pi n(k-l)}{L} - \cos\frac{2\pi n(k+l)}{L} \right) . 
\end{aligned} 
\end{equation} 

which yields
\begin{equation} 
    \cov(\bvareps_{i,k}, \bvareps_{i,l}) 
    = \frac{\sigma^2}{2}\left[\delta_L(k-l)-\delta_L(k+l)\right].
\label{eq:cov_I_I} 
\end{equation}

These identities reveal the Hermitian degeneracies explicitly.
For the DC bin $k=0$, and also the Nyquist bin $k=L/2$ when $L$ is even, the imaginary variance vanishes:
\begin{equation}
    \operatorname{Var}(\bvareps_{i,0}) = 0, \quad
    \operatorname{Var}(\bvareps_{i,L/2}) = 0 \quad (L \text{ even}).
\end{equation}
The corresponding real variances equal $\sigma^2$:
\begin{equation}
    \operatorname{Var}(\bvareps_{r,0}) = \sigma^2, \quad
    \operatorname{Var}(\bvareps_{r,L/2}) = \sigma^2 \quad (L \text{ even}).
\end{equation}
For non-self-conjugate frequency pairs, both quadratures have variance $\sigma^2/2$ and conjugate bins are anti-diagonally correlated.
Thus, $\bSigma_r$ and $\bSigma_i$ are heteroscedastic and singular on the redundant full grid.
All Mahalanobis norms, precision matrices, and determinants involving $\bSigma_r$ and $\bSigma_i$ are taken on the independent Hermitian coordinate chart, equivalently as the support inverse and pseudo-determinant of the full-grid degenerate Gaussian.

\subsection{Covariance Matrix Visualization}
The derived covariance structures $\bSigma_r$ and $\bSigma_i$ are sparse matrices characterized by an ``X-shape'' structure (diagonal + anti-diagonal), reflecting Hermitian symmetry.
The following matrices are expanded visualizations of the compact covariance identities in Eqs.~\eqref{eq:cov_R_R} and \eqref{eq:cov_I_I}.
% These expanded matrices are schematic for nontrivial lengths (even $L \ge 4$ and odd $L \ge 3$); small boundary cases are still governed by the compact identities above.

For an even $L$, the covariance matrices $\bSigma_r$ and $\bSigma_i$ are:
\begin{equation}
    \underset{L \in \{2n | n \in \N \}}{\bSigma_r} = \sigma^2
    \begin{bmatrix}
        1      &0           &0      &\cdots      &0      &\cdots      &0      &0           \\ 
        0      &\frac{1}{2} &0      &\cdots      &0      &\cdots      &0      &\frac{1}{2} \\  
        0      &0           &\ddots &            &\vdots &            &\idots &0           \\
        \vdots &\vdots      &       &\frac{1}{2} &0      &\frac{1}{2} &       &\vdots      \\
        0      &0           &\cdots &0           &1      &0           &\cdots &0           \\
        \vdots &\vdots      &       &\frac{1}{2} &0      &\frac{1}{2} &       &\vdots      \\
        0      &0           &\idots &            &\vdots &            &\ddots &0           \\
        0      &\frac{1}{2} &0      &\cdots      &0      &\cdots      &0      &\frac{1}{2}
    \end{bmatrix}_{L \times L}
    , \quad
    \underset{L \in \{2n | n \in \N \}}{\bSigma_i} = \sigma^2
    \begin{bmatrix}
        0      &0            &0      &\cdots       &0      &\cdots       &0      &0            \\ 
        0      &\frac{1}{2}  &0      &\cdots       &0      &\cdots       &0      &-\frac{1}{2} \\  
        0      &0            &\ddots &             &\vdots &             &\idots &0            \\
        \vdots &\vdots       &       &\frac{1}{2}  &0      &-\frac{1}{2} &       &\vdots       \\
        0      &0            &\cdots &0            &0      &0            &\cdots &0            \\
        \vdots &\vdots       &       &-\frac{1}{2} &0      &\frac{1}{2}  &       &\vdots       \\
        0      &0            &\idots &             &\vdots &             &\ddots &0            \\
        0      &-\frac{1}{2} &0      &\cdots       &0      &\cdots       &0      &\frac{1}{2}
    \end{bmatrix}_{L \times L}
\label{eq:even_sigma}
\end{equation}

For an odd $L$, the covariance matrices $\bSigma_r$ and $\bSigma_i$ are:

\begin{equation}
    \underset{L \in \{2n+1 | n \in \N \}}{\bSigma_r} = \sigma^2
    \begin{bmatrix}
        1      &0           &0      &\cdots      &\cdots      &0      &0           \\ 
        0      &\frac{1}{2} &0      &\cdots      &\cdots      &0      &\frac{1}{2} \\  
        0      &0           &\ddots &            &            &\idots &0           \\
        \vdots &\vdots      &       &\frac{1}{2} &\frac{1}{2} &       &\vdots      \\
        \vdots &\vdots      &       &\frac{1}{2} &\frac{1}{2} &       &\vdots      \\
        0      &0           &\idots &            &            &\ddots &0           \\
        0      &\frac{1}{2} &0      &\cdots      &\cdots      &0      &\frac{1}{2}
    \end{bmatrix}_{L \times L}
    , \quad
    \underset{L \in \{2n+1 | n \in \N \}}{\bSigma_i} = \sigma^2
    \begin{bmatrix}
        0      &0            &0      &\cdots       &\cdots       &0      &0            \\ 
        0      &\frac{1}{2}  &0      &\cdots       &\cdots       &0      &-\frac{1}{2} \\  
        0      &0            &\ddots &             &             &\idots &0            \\
        \vdots &\vdots       &       &\frac{1}{2}  &-\frac{1}{2} &       &\vdots       \\
        \vdots &\vdots       &       &-\frac{1}{2} &\frac{1}{2}  &       &\vdots       \\
        0      &0            &\idots &             &             &\ddots &0            \\
        0      &-\frac{1}{2} &0      &\cdots       &\cdots       &0      &\frac{1}{2}
    \end{bmatrix}_{L \times L}
\label{eq:odd_sigma}
\end{equation}

\subsection{Closedness under Gaussian Superposition}
\label{app:additivity}
The forward transition relies on the fact that weighted sums of independent Gaussian noises remain Gaussian.
The same closedness holds after DFT projection, provided that the spectral noise is understood as the Hermitian-constrained Gaussian induced by real temporal noise.

\begin{tcolorbox}[tcbset]
\begin{theorem}
\label{thm:add_unitary_inv}
    \textbf{(Closedness of Hermitian Spectral Gaussian Noise).}
    Let $\beps_1,\dots,\beps_T \in \R^L$ be independent temporal Gaussian noise vectors with $\beps_t \sim \mcN(\mbzero,\sigma^2\mbI)$, and let $\mcE_t=\mcF(\beps_t)$.
    For any real weights $\{w_t\}_{t=1}^{T}$, define $c=(\sum_{t=1}^{T}w_t^2)^{1/2}\ge 0$.
    Then
    \begin{equation}
        \sum_{t=1}^{T} w_t \mcE_t
        \stackrel{d}{=}
        c\,\mcE^*,
    \end{equation}
    where $\mcE^*=\mcF(\beps^*)$ and $\beps^* \sim \mcN(\mbzero,\sigma^2\mbI)$.
    Therefore, Hermitian spectral Gaussian noise is closed under the weighted Gaussian superposition used by DDPM forward reparameterization.
\end{theorem}
\end{tcolorbox}

\begin{proof}
Let
\begin{equation}
    \beps_{\mathrm{sum}} = \sum_{t=1}^{T} w_t \beps_t .
\end{equation}
Because the $\beps_t$ are independent centered Gaussians,
\begin{equation}
    \beps_{\mathrm{sum}}
    \sim \mcN\left(\mbzero, \sigma^2 \sum_{t=1}^{T} w_t^2 \mbI\right)
    =
    \mcN(\mbzero, c^2\sigma^2\mbI).
\end{equation}
Thus, $\beps_{\mathrm{sum}} \stackrel{d}{=} c\beps^*$ for $\beps^* \sim \mcN(\mbzero,\sigma^2\mbI)$.
Applying the linear DFT operator gives
\begin{equation}
    \sum_{t=1}^{T} w_t \mcE_t
    =
    \mcF\left(\sum_{t=1}^{T} w_t \beps_t\right)
    =
    \mcF(\beps_{\mathrm{sum}})
    \stackrel{d}{=}
    c\mcF(\beps^*)
    =
    c\mcE^* .
\end{equation}
Since $\beps^*$ is real-valued, $\mcE^*$ satisfies the same Hermitian constraint and covariance identities derived above.
For the DDPM marginal transition, the cumulative noise weights satisfy $c^2=1-\bar{\alpha}_t$, yielding
\begin{equation}
    \mcE_{\text{cumulative}} 
    \stackrel{d}{=}
    \sqrt{1-\bar{\alpha}_t}\,\mcE^*.
\end{equation}

\end{proof}

\newpage
\section{Proof of Proposition~\ref{prop:cond_reverse_factor} (Conditional Reverse Factorization)}
\label{app:cond_reverse_factor}

The proof has two steps.
We first establish the conditional factorization of the forward trajectory under a fixed initial state $\mcX_0$.
We then substitute this factorization into Bayes' theorem to obtain the reverse posterior factorization in Proposition~\ref{prop:cond_reverse_factor}.

\subsection{Conditional Trajectory Factorization}
\label{app:proof_reverse_indep_algebraic}

\begin{tcolorbox}[tcbset]
\begin{lemma}
\label{lem:cond_factor_traj}
    \textbf{(Conditional Factorizability of Complex Trajectories).}
    Given the quadrature independence of complex frequency noise, the true conditional distribution $q(\mcX_t | \mcX_0)$ factorizes into independent marginals of its real and imaginary parts for any fixed initial boundary $\mcX_0$:
    \begin{equation}
        q(\mcX_t | \mcX_0)
        = q(\mcR_t | \mcR_0) q(\mcI_t | \mcI_0) .
    \end{equation}
\end{lemma}
\end{tcolorbox}

\begin{proof}
For a complex spectral state $\mcX_t = \mcR_t + j \mcI_t \in \C^L$, represent the variable by the augmented real vector $\mcX_{\mathrm{aug}}=[\mcR_t^\top,\mcI_t^\top]^\top$.
Conditioned on a fixed initial state $\mcX_0$, the DFT-domain forward marginal can be written as:
\begin{align}
    \mcR_t
    &= \sqrt{\bar{\alpha}_t}\mcR_0+\sqrt{1-\bar{\alpha}_t}\bvareps_r,
    \quad \bvareps_r \sim \mcN(\mbzero,\bSigma_r),
    \\
    \mcI_t
    &= \sqrt{\bar{\alpha}_t}\mcI_0+\sqrt{1-\bar{\alpha}_t}\bvareps_i,
    \quad \bvareps_i \sim \mcN(\mbzero,\bSigma_i).
\end{align}
By Theorem~\ref{thm:freq_ortho}, $\bvareps_r$ and $\bvareps_i$ are statistically independent.
Therefore, the augmented conditional Gaussian has mean and covariance:
\begin{equation}
    \bmu_{\text{aug}}
    =
    \sqrt{\bar{\alpha}_t}
    \begin{bmatrix}
        \mcR_0\\
        \mcI_0
    \end{bmatrix},
    \quad
    \bSigma_{\text{aug}} =
    \begin{bmatrix} 
        (1-\bar{\alpha}_t)\bSigma_r & \mbzero\\
        \mbzero & (1-\bar{\alpha}_t)\bSigma_i
    \end{bmatrix} .
\end{equation}
The off-diagonal covariance blocks vanish exactly because $\cov(\bvareps_r,\bvareps_i)=\mbzero$.
Equivalently, define
\begin{equation}
    \bmu_t^{\mcR}=\sqrt{\bar{\alpha}_t}\mcR_0,\quad
    \bmu_t^{\mcI}=\sqrt{\bar{\alpha}_t}\mcI_0,\quad
    \bSigma_t^{\mcR}=(1-\bar{\alpha}_t)\bSigma_r,\quad
    \bSigma_t^{\mcI}=(1-\bar{\alpha}_t)\bSigma_i .
\end{equation}
The block-diagonal covariance of this joint Gaussian separates the quadratic form and the normalizing constant.
Using the support-wise determinant and inverse convention from Appendix~\ref{app:spec_dist_gauss}, the Gaussian density with respect to the induced Hermitian coordinate measure becomes:
\begin{equation}
\begin{aligned}
    q(\mcX_t | \mcX_0) 
    &= C_t
    \exp \left( 
        -\frac{1}{2} \left[ 
            (\mcR_t - \bmu^{\mcR}_t)^{\top} (\bSigma^{\mcR}_t)^{-1} (\mcR_t - \bmu^{\mcR}_t) + 
            (\mcI_t - \bmu^{\mcI}_t)^{\top} (\bSigma^{\mcI}_t)^{-1} (\mcI_t - \bmu^{\mcI}_t) 
        \right] 
    \right)
    \\
    &= \prod_{\zeta \in \{ \mcR, \mcI \}} 
    C_{\zeta,t}
    \exp \left( 
        -\frac{1}{2} 
        (\zeta_t - \bmu^{\zeta}_t)^{\top} 
        (\bSigma^{\zeta}_t)^{-1} 
        (\zeta_t - \bmu^{\zeta}_t) 
    \right)
    \\
    &= q(\mcR_t | \mcR_0) \cdot q(\mcI_t | \mcI_0)
\end{aligned}
\end{equation}
where $C_t=C_{\mcR,t}C_{\mcI,t}$ collects the corresponding restricted Gaussian normalizing constants.
This proves the desired conditional factorization.
\end{proof}

\subsection{Proof of the Reverse Posterior Factorization}

The reverse denoising posterior in DDPM is
\begin{equation}
    q(\mcX_{t-1} | \mcX_t, \mcX_0)
    =
    \frac{
        q(\mcX_t | \mcX_{t-1}, \mcX_0)
        q(\mcX_{t-1} | \mcX_0)
    }{
        q(\mcX_t | \mcX_0)
    } .
\label{eq:bayes_posterior}
\end{equation}
By the Markov property, $q(\mcX_t | \mcX_{t-1},\mcX_0)=q(\mcX_t|\mcX_{t-1})$.
The one-step forward transition factorizes because the injected quadrature noises are independent:
\begin{equation}
    q(\mcX_t | \mcX_{t-1})
    =
    q(\mcR_t | \mcR_{t-1})
    q(\mcI_t | \mcI_{t-1}) .
\end{equation}
Lemma~\ref{lem:cond_factor_traj} further gives the conditional marginals:
\begin{equation}
    q(\mcX_s|\mcX_0)
    =
    q(\mcR_s|\mcR_0)
    q(\mcI_s|\mcI_0),
    \quad s\in\{t-1,t\}.
\end{equation}
Substituting these three factorizations into Eq.~\eqref{eq:bayes_posterior} yields:
\begin{equation} 
\begin{aligned} 
    q(\mcX_{t-1} | \mcX_t, \mcX_0) 
    &=
    \frac{
        q(\mcR_t|\mcR_{t-1})
        q(\mcI_t|\mcI_{t-1})
        q(\mcR_{t-1}|\mcR_0)
        q(\mcI_{t-1}|\mcI_0)
    }{
        q(\mcR_t|\mcR_0)q(\mcI_t|\mcI_0)
    }
    \\
    &=
    \underbrace{
        \frac{q(\mcR_t|\mcR_{t-1})q(\mcR_{t-1}|\mcR_0)}
        {q(\mcR_t|\mcR_0)}
    }_{q(\mcR_{t-1}|\mcR_t,\mcR_0)}
    \cdot
    \underbrace{
        \frac{q(\mcI_t|\mcI_{t-1})q(\mcI_{t-1}|\mcI_0)}
        {q(\mcI_t|\mcI_0)}
    }_{q(\mcI_{t-1}|\mcI_t,\mcI_0)} .
\end{aligned}
\end{equation}
Therefore,
\begin{equation}
    q(\mcX_{t-1} | \mcX_t, \mcX_0)
    =
    q(\mcR_{t-1}|\mcR_t,\mcR_0)
    q(\mcI_{t-1}|\mcI_t,\mcI_0),
\end{equation}
which proves Proposition~\ref{prop:cond_reverse_factor}.

\subsection{Post-proof Transition to Mean Field Approximation} 
\label{sec:transition_to_mft_proof}

Proposition~\ref{prop:cond_reverse_factor} concerns the true posterior conditioned on the initial state $\mcX_0$.
The learned generative transition, however, depends only on $\mcX_t$ and therefore requires marginalizing over the unknown initial state:
\begin{equation}
    q(\mcX_{t-1} | \mcX_t)
    =
    \int
    q(\mcX_{t-1} | \mcX_t,\mcX_0)
    q(\mcX_0|\mcX_t)
    \diff \mcX_0 .
\end{equation}
Although the conditional kernel inside the integral factorizes, the posterior $q(\mcX_0|\mcX_t)$ generally does not, because $\mcR_0$ and $\mcI_0$ are coupled by the data manifold (\eg phase coherence).
Thus, the marginal reverse transition remains entangled.

This motivates the Mean Field Theory (MFT) approximation in Section~\ref{sec:mft_parallel}.
Using a factorized variational proxy for the unknown initial-state posterior, the reverse transition is approximated by:
\begin{equation} 
\begin{aligned} 
    p_{\theta}(\mcX_{t-1} | \mcX_t) 
    &\approx
    \int
    q(\mcR_{t-1}|\mcR_t,\mcR_0)
    q(\mcI_{t-1}|\mcI_t,\mcI_0)
    p_{\theta_r}(\mcR_0|\mcR_t)
    p_{\theta_i}(\mcI_0|\mcI_t)
    \diff \mcR_0 \diff \mcI_0
    \\ 
    &=
    \left[
        \int q(\mcR_{t-1}|\mcR_t,\mcR_0)
        p_{\theta_r}(\mcR_0|\mcR_t)
        \diff \mcR_0
    \right]
    \left[
        \int q(\mcI_{t-1}|\mcI_t,\mcI_0)
        p_{\theta_i}(\mcI_0|\mcI_t)
        \diff \mcI_0
    \right]
    \\ 
    &=
    p_{\theta_r}(\mcR_{t-1}|\mcR_t)
    p_{\theta_i}(\mcI_{t-1}|\mcI_t) .
\end{aligned} 
\end{equation}
This explains why the conditional theory supports parallel branches, while the practical model still needs the interactive correction branch to compensate for correlations lost under the MFT approximation.

\newpage
\section{Reverse Dynamics under Heteroscedastic Gaussian Noise}
\label{app:hetero_reverse_derivation}

This appendix derives the DDPM reverse posterior for $\vecx \in \R^{d}$ under Gaussian noise $\beps \sim \mcN(\mbzero, \bSigma)$ with arbitrary symmetric positive definite covariance $\bSigma$.
The result justifies the Mahalanobis training objective: the posterior keeps the same covariance shape up to a scalar schedule factor, and the variational KL reduces to a precision-weighted noise prediction error.
In PaCoDi, we apply this result to the real and imaginary spectral branches with $\bSigma=\bSigma_r$ and $\bSigma=\bSigma_i$, respectively; throughout, $q$ denotes the true forward or posterior distribution and $p_{\theta}$ the learned reverse model.

\subsection{Bayesian Formulation of the Reverse Step}
Using \textbf{Bayes' Theorem}, the true reverse posterior $q(\vecx_{t-1} | \vecx_t, \vecx_0)$ is given by:
\begin{equation}
    q(\vecx_{t-1} | \vecx_t, \vecx_0)
    = \frac{q(\vecx_t | \vecx_{t-1}) q(\vecx_{t-1} | \vecx_0)}{q(\vecx_t | \vecx_0)}.
\end{equation}

Recall the forward diffusion transitions defined by the schedule $\alpha_t = 1 - \beta_t$ and $\bar{\alpha}_t = \prod_{s=1}^t \alpha_s$:
\begin{gather}
    q(\vecx_t | \vecx_{t-1}) 
    = \mcN(\vecx_t; \sqrt{\alpha_t} \vecx_{t-1}, (1-\alpha_t)\bSigma),
    \\
    q(\vecx_t | \vecx_0) 
    = \mcN(\vecx_t; \sqrt{\bar{\alpha}_t} \vecx_0, (1 - \bar{\alpha}_t)\bSigma).
\end{gather}

Substituting these densities into the Bayes' rule ratio:
\begin{equation}
    q(\vecx_{t-1} | \vecx_t, \vecx_0) \propto
    \exp \left(
        - \frac{1}{2} \left[
            (\vecx_t - \sqrt{\alpha_t} \vecx_{t-1})^\top \frac{\bSigma^{-1}}{1-\alpha_t} (\vecx_t - \sqrt{\alpha_t} \vecx_{t-1})
            + (\vecx_{t-1} - \sqrt{\bar{\alpha}_{t-1}} \vecx_0)^\top \frac{\bSigma^{-1}}{1-\bar{\alpha}_{t-1}} (\vecx_{t-1} - \sqrt{\bar{\alpha}_{t-1}} \vecx_0)
            - \mathcal{C}(\vecx_t, \vecx_0)
        \right]
    \right)
\end{equation}
where $\mathcal{C}(\vecx_t, \vecx_0)$ denotes terms independent of $\vecx_{t-1}$.

\subsection{Completing the Square}
To identify the mean $\tilde{\bmu}_t$ and covariance $\tilde{\bSigma}_t$ of the resulting Gaussian $q(\vecx_{t-1} | \vecx_t, \vecx_0)$, we collect all terms quadratic and linear in $\vecx_{t-1}$.

\paragraph{1. The Precision Matrix (Inverse Covariance):}
Extracting the quadratic terms $\vecx_{t-1}^\top (\cdot) \vecx_{t-1}$:
\begin{equation}
    \tilde{\bSigma}_t^{-1}
    = \frac{\alpha_t}{1-\alpha_t} \bSigma^{-1} + \frac{1}{1-\bar{\alpha}_{t-1}} \bSigma^{-1}
    = \left( \frac{1 - \bar{\alpha}_t}{(1-\alpha_t)(1-\bar{\alpha}_{t-1})} \right) \bSigma^{-1}.
\end{equation}

Inverting this yields the reverse covariance:
\begin{equation}
\tilde{\bSigma}_t 
= \frac{(1-\alpha_t)(1-\bar{\alpha}_{t-1})}{1-\bar{\alpha}_t} \bSigma .
\end{equation}

This result confirms that the shape of $\bSigma$ is preserved throughout the reverse process, merely scaled by a time-dependent scalar.

\paragraph{2. The Mean Vector:}
Extracting the linear terms $-2 \tilde{\bmu}_t^\top \tilde{\bSigma}_t^{-1} \vecx_{t-1}$ involves the matrix $\bSigma^{-1}$:
\begin{equation}
    \tilde{\bmu}_t^\top \tilde{\bSigma}_t^{-1}
    = \frac{\sqrt{\alpha_t}}{1-\alpha_t} \vecx_t^\top \bSigma^{-1} 
    + \frac{\sqrt{\bar{\alpha}_{t-1}}}{1-\bar{\alpha}_{t-1}} \vecx_0^\top \bSigma^{-1} .
\end{equation} 

Multiplying from the right by $\tilde{\bSigma}_t$:
\begin{equation}
\begin{aligned}
    \tilde{\bmu}_t
    &= \tilde{\bSigma}_t \bSigma^{-1} \left( 
        \frac{\sqrt{\alpha_t}}{1-\alpha_t} \vecx_t 
        + \frac{\sqrt{\bar{\alpha}_{t-1}}}{1-\bar{\alpha}_{t-1}} \vecx_0 
    \right)
    \\
    &= \left( 
        \frac{(1-\alpha_t)(1-\bar{\alpha}_{t-1})}{1-\bar{\alpha}_t} \bSigma \bSigma^{-1} 
    \right) \left( 
        \frac{\sqrt{\alpha_t}}{1-\alpha_t} \vecx_t + \frac{\sqrt{\bar{\alpha}_{t-1}}}{1-\bar{\alpha}_{t-1}} \vecx_0 
    \right)
    \\
    &= \frac{\sqrt{\alpha_t}(1-\bar{\alpha}_{t-1})}{1-\bar{\alpha}_t} \vecx_t 
    + \frac{\sqrt{\bar{\alpha}_{t-1}}(1-\alpha_t)}{1-\bar{\alpha}_t} \vecx_0 .
\end{aligned}
\end{equation}
Substituting $\vecx_0 = \frac{1}{\sqrt{\bar{\alpha}_t}}(\vecx_t - \sqrt{1 - \bar{\alpha}_t} \beps)$ into the mean equation allows us to express it in terms of the cumulative noise $\beps$:
\begin{equation}
    \tilde{\bmu}_t(\vecx_t, \beps) 
    = \frac{1}{\sqrt{\alpha_t}} \left( 
        \vecx_t - \frac{1 - \alpha_t}{\sqrt{1 - \bar{\alpha}_t}} \beps 
    \right) .
\end{equation}

Thus, the true posterior is: 
\begin{equation} 
    q(\vecx_{t-1} | \vecx_t, \vecx_0) 
    = \mcN \left( \vecx_{t-1}; \tilde{\bmu}_t(\vecx_t, \beps), \tilde{\bSigma}_t \right) . 
\end{equation}

\subsection{Optimization Objective}
We approximate the true process using a learned model $p_{\theta}(\vecx_{t-1} | \vecx_t)$ that predicts the noise $\beps_{\theta}(\vecx_t, t)$. The distribution is parameterized as:
\begin{equation}
    p_{\theta}(\vecx_{t-1} | \vecx_t) 
    = \mcN \left( \vecx_{t-1}; \tilde{\bmu}_t(\vecx_t, \beps_{\theta}), \tilde{\bSigma}_t \right) .
\end{equation}

The training objective minimizes the KL divergence between the true posterior and the parameterized reverse process, averaged over the sampled timestep and noisy state.
The KL divergence itself integrates over $\vecx_{t-1}$; the outer expectation below is only over $(t,\vecx_0,\vecx_t)$.
Since both conditional distributions are Gaussian with identical covariance $\tilde{\bSigma}_t$, the KL divergence reduces to the Mahalanobis distance between their means:
\begin{equation}
\begin{aligned}
    \mathcal{L}
    &= \expt_{t,\vecx_0,\vecx_t \sim q} \left[
        D_{\text{KL}} \left(
            q(\vecx_{t-1} |\vecx_t, \vecx_0)
            \parallel
            p_{\theta}(\vecx_{t-1}| \vecx_t)
        \right)
    \right]
    \\
    &= \expt_{t,\vecx_0,\vecx_t \sim q} \left[ 
        \frac{1}{2} 
        (\tilde{\bmu}_t(\beps) - \tilde{\bmu}_t(\beps_{\theta}))^{\top} 
        \tilde{\bSigma}_t^{-1} 
        (\tilde{\bmu}_t(\beps) - \tilde{\bmu}_t(\beps_{\theta})) 
    \right] .
\end{aligned}
\end{equation}

Substituting the expressions for $\tilde{\bmu}_t$ difference and $\tilde{\bSigma}_t^{-1}$:
\begin{equation}
    \tilde{\bmu}_t(\beps) - \tilde{\bmu}_t(\beps_{\theta})
    = - \frac{1 - \alpha_t}{\sqrt{\alpha_t} \sqrt{1 - \bar{\alpha}_t}} (\beps - \beps_{\theta})
    , \quad
    \tilde{\bSigma}_t^{-1}
    = \left( \frac{1 - \bar{\alpha}_t}{(1 - \alpha_t)(1 - \bar{\alpha}_{t-1})} \right) \bSigma^{-1} .
\end{equation}

The loss simplifies to:
\begin{equation}
    \mathcal{L}
    = \lambda(t) \cdot (\beps - \beps_{\theta})^{\top} \bSigma^{-1} (\beps - \beps_{\theta}),
\end{equation}
where $\lambda(t)$ is a scalar weighting term dependent on the diffusion schedule.

Ignoring the weighting factors, the core optimization problem minimizes the Mahalanobis norm of the noise prediction error: 
\begin{equation} 
    \mathcal{L}_{\text{simple}} = 
    \expt_{\vecx_0, \beps, t} \left[ 
        \left\Vert \beps - \beps_{\theta}(\vecx_t, t) \right\Vert_{\bSigma^{-1}}^2 
    \right] . 
\end{equation}

This confirms that for any Gaussian noise process with a general covariance matrix $\bSigma$ (whether diagonal or not), the optimal loss function corresponds to the squared error weighted by the precision matrix $\bSigma^{-1}$. This effectively whitens the residuals during training.

\newpage
\section{Frequency Stochastic Differential Equation Framework}

\subsection{Frequency SDEs from Discrete DDPM}
\label{app:derivation_SDE}

This section derives the continuous-time limit of the discrete PaCoDi forward diffusion transition and obtains the corresponding frequency-domain VP-SDE.

\subsubsection{Discrete-Time Evolution with Time Scaling}
Let the diffusion process be indexed by steps $i \in \{0,1,\dots,N\}$ and embed them in $t \in [0,1]$ with $\Delta t = 1/N$ and $t_i=i\Delta t$.

We scale the discrete noise schedule as
\begin{equation}
    \beta_i \coloneqq \beta(t_i) \Delta t
\label{eq:beta_scaling}
\end{equation}

Substituting this scaling into the DDPM transition gives
\begin{equation}
    \mcX_{t+\Delta t}
    = \sqrt{1 - \beta(t)\Delta t} \mcX_t + \sqrt{\beta(t)\Delta t} \mcE
\label{eq:scaled_transition}
\end{equation}
where $\mcX_t \in \C^{L}$ is the frequency state and $\mcE=\bvareps_r+j\bvareps_i$ is the Hermitian complex Gaussian noise with $\bvareps_r \sim \mcN(\mbzero,\bSigma_r)$ and $\bvareps_i \sim \mcN(\mbzero,\bSigma_i)$.

Using the Taylor expansion
\begin{equation}
    \sqrt{1 - \beta(t)\Delta t}
    = 1 - \frac{1}{2} \beta(t) \Delta t + \mcO((\Delta t)^2)
\end{equation}

Eq.~\eqref{eq:scaled_transition} yields the increment form
\begin{equation}
    \mcX_{t+\Delta t}
    \approx \left( 1 - \frac{1}{2} \beta(t) \Delta t \right) \mcX_t + \sqrt{\beta(t)} \sqrt{\Delta t} \mcE
    \quad \Rightarrow \quad
    \mcX_{t+\Delta t} - \mcX_t
    \approx - \frac{1}{2} \beta(t) \mcX_t \Delta t + \sqrt{\beta(t)} \underbrace{ \mcE \sqrt{\Delta t} }_{\Delta \mcW_t} .
\end{equation}

\subsubsection{Continuous-Time Extension and the Spectral Wiener Process}
We use $\mcW$ to denote the Spectral Wiener Process, \ie heteroscedastic Spectral Wiener noise obtained by applying the DFT to a real temporal Wiener process. The increment induced by the scaled DDPM noise is
\begin{equation}
    \Delta \mcW 
    \coloneqq \sqrt{\Delta t}\mcE
    = \sqrt{\Delta t}\bvareps_r + j\sqrt{\Delta t}\bvareps_i
    = \Delta\bw_r+j\Delta\bw_i ,
\end{equation}
where $\bvareps_r \sim \mcN(\mbzero,\bSigma_r)$ and $\bvareps_i \sim \mcN(\mbzero,\bSigma_i)$, with $\bSigma_r$ and $\bSigma_i$ defined in Eqs.~\eqref{eq:even_sigma} and \eqref{eq:odd_sigma}. In the limit $\Delta t\to0$, this gives
\begin{equation}
    \diff \mcW = \diff \bw_r + j \diff \bw_i
\label{eq:diff_of_wiener}
\end{equation}
with heteroscedastic quadrature covariances
\begin{equation}
    \expt [\diff \bw_r \diff \bw_r^{\top}] = \bSigma_r \diff t , \quad
    \expt [\diff \bw_i \diff \bw_i^{\top}] = \bSigma_i \diff t , \quad
    \expt [\diff \bw_r \diff \bw_i^{\top}] = \mbzero .
\end{equation}
Equivalently, this Spectral Wiener noise is induced by the temporal Wiener increment through $\diff\mcW=\mcF(\diff\mbw)$.

\subsubsection{Forward Stochastic Differential Equation}
Taking $\Delta t\to0$ in the increment equation gives the frequency-domain It\^o SDE
\begin{equation}
    \diff \mcX 
    = \underbrace{- \frac{1}{2} \beta(t) \mcX}_{\mbf(\mcX, t)} \diff t + \underbrace{\sqrt{\beta(t)} \mbI}_{\mbg(t)} \diff \mcW
\label{eq:freq_SDE}
\end{equation}
where $\mbf(\mcX, t) = -\frac{1}{2}\beta(t)\mcX$ and $\mbg(t)=\sqrt{\beta(t)}\mbI$ denote the drift and diffusion coefficients of a VP-SDE. 
This derivation connects the discrete PaCoDi DDPM transition with the continuous frequency SDE framework presented in Section~\ref{sec:continuous_sde}.

\subsection{Parallel Quadrature Reverse SDE}
\label{app:reverse_SDE}

In this section, we derive the reverse-time stochastic dynamics and prove that, under conditional quadrature factorization, the complex reverse SDE naturally decouples into two parallel real-valued processes.

\subsubsection{General Conditional Reverse-Time SDE}
The covariance term in the reverse drift is inherited from the forward Spectral Wiener noise. Since $\diff \mcW=\diff \bw_r+j\diff \bw_i$ has quadrature covariances $\bSigma_r\diff t$ and $\bSigma_i\diff t$ with zero cross-covariance, we denote its covariance operator by $\bSigma_{\mcX}$. Equivalently, for any $u,v\in\R^L$,
\begin{equation}
    \bSigma_{\mcX}(u+jv) = \bSigma_r u + j \bSigma_i v .
\end{equation}

Thus the local covariance of the diffusion term $\mbg(t)\diff\mcW$ is $\mbg(t)\bSigma_{\mcX}\mbg(t)^\top \diff t$. Instantiating the reverse-time diffusion theory for the conditional transition density gives
\begin{equation}
    \diff \mcX
    = \left[
        \mbf(\mcX, t) - \mbg(t) \bSigma_{\mcX} \mbg(t)^{\top} \nabla_{\bar{\mcX}} \log p_{t|0}(\mcX)
    \right] \diff t + \mbg(t) \diff \hat{\mcW}
\label{eq:reverse_SDE}
\end{equation}
where $\diff \hat{\mcW}$ denotes the reverse-time Spectral Wiener increment induced by the same DFT mapping, and $\nabla_{\bar{\mcX}} \log p_{t|0}(\mcX)$ represents the conditional \textit{Stein score function} defined over the complex manifold.

\subsubsection{Score Decomposition via Wirtinger Calculus}
We focus on the conditional reverse process given a fixed initial state $\mcX_0$. As proven in \textbf{Lemma~\ref{lem:cond_factor_traj}} (Appendix~\ref{app:cond_reverse_factor}), the conditional density factorizes due to quadrature independence:
\begin{equation}
    p_{t|0}(\mcX) \coloneqq p(\mcX_t | \mcX_0) 
    = p_{t|0}(\mcR) p_{t|0}(\mcI) .
\end{equation}

To compute the gradient of this factorized real-valued probability density with respect to the complex variable $\mcX_t$, we employ \textbf{Wirtinger Calculus}. The conjugate gradient operator is defined as $\nabla_{\bar{\mcX}} \coloneqq \frac{1}{2} (\nabla_{\mcR} + j \nabla_{\mcI})$. Applying this to the log-density:
\begin{equation}
\begin{aligned}
    \nabla_{\bar{\mcX}} \log p_{t|0}(\mcX)
    &= \frac{1}{2} \left(
        \frac{\partial}{\partial \mcR} + j \frac{\partial}{\partial \mcI}
    \right) \left[ \log p_{t|0}(\mcR) + \log p_{t|0}(\mcI) \right]
    \\
    &= \frac{1}{2} \left(
        \underbrace{\frac{\partial \log p_{t|0}(\mcR)}{\partial \mcR}}_{\nabla_{\mcR} \log p_{t|0}(\mcR)}
        + \underbrace{\frac{\partial \log p_{t|0}(\mcI)}{\partial \mcR}}_{0}
        + j \underbrace{\frac{\partial \log p_{t|0}(\mcR)}{\partial \mcI}}_{0}
        + j \underbrace{\frac{\partial \log p_{t|0}(\mcI)}{\partial \mcI}}_{\nabla_{\mcI} \log p_{t|0}(\mcI)}
    \right)
    \\
    &= \frac{1}{2} \left(
        \nabla_{\mcR} \log p_{t|0}(\mcR)
        + j \nabla_{\mcI} \log p_{t|0}(\mcI)
    \right).
\end{aligned}
\label{eq:nabla_separation}
\end{equation}

This derivation explicitly shows that the cross-terms vanish, and the factor $\frac{1}{2}$ emerges naturally from the complex derivative definition.

\subsubsection{Decoupling of the SDE}
Since the noise schedule $\beta(t)$ is real-valued, the linear drift coefficient $\mbf(\mcX, t) = -\frac{1}{2}\beta(t)\mcX_t$ satisfies the additive property:
\begin{equation}
    \mbf(\mcX, t) = \mbf(\mcR, t) + j \mbf(\mcI, t).
\label{eq:drift_coeff_separation}
\end{equation}

For any decomposed score $s_r+js_i$, the covariance operator acts component-wise as
\begin{equation}
    \bSigma_{\mcX}(s_r+js_i)=\bSigma_r s_r+j\bSigma_i s_i .
\end{equation}
Substituting the decomposed score (Eq.~\eqref{eq:nabla_separation}), the drift (Eq.~\eqref{eq:drift_coeff_separation}), this component-wise covariance action, and the reverse-time Spectral Wiener increment $\diff \hat{\mcW} = \diff \hat{\bw}_r + j \diff \hat{\bw}_i$ into the general reverse SDE (Eq.~\eqref{eq:reverse_SDE}), we obtain:
\begin{equation}
    \diff (\mcR + j \mcI)
    = \left[
        ( \mbf(\mcR, t) + j \mbf(\mcI, t) )
        - \mbg(t) \bSigma_{\mcX} \mbg(t)^{\top} \frac{1}{2} \left( \nabla_{\mcR} \log p_{t|0}(\mcR) + j \nabla_{\mcI} \log p_{t|0}(\mcI) \right)
    \right] \diff t + \mbg(t) (\diff \hat{\bw}_r + j \diff \hat{\bw}_i).
\end{equation}

By grouping real and imaginary terms, the complex SDE separates into two autonomous real-valued systems:
\begin{equation}
\begin{dcases}
    \diff \mcR
    = \left[
        \mbf(\mcR, t) - \frac{1}{2} \mbg(t) \bSigma_r \mbg(t)^{\top} \nabla_{\mcR} \log p_{t|0}(\mcR)
    \right] \diff t + \mbg(t) \diff \hat{\bw}_r
    \\
    \diff \mcI
    = \left[
        \mbf(\mcI, t) - \frac{1}{2} \mbg(t) \bSigma_i \mbg(t)^{\top} \nabla_{\mcI} \log p_{t|0}(\mcI)
    \right] \diff t + \mbg(t) \diff \hat{\bw}_i
\end{dcases}
\end{equation}
where $\diff \hat{\bw}_r$ and $\diff \hat{\bw}_i$ are reverse-time quadrature Brownian increments with heteroscedastic spectral covariances.

\paragraph{Conclusion.}
This result proves \textbf{Theorem~\ref{thm:reverse_SDE}}: under the conditional score factorization, the reverse dynamics separate into two parallel real-valued SDEs. The practical marginal reverse process is then handled by the MFT approximation and interactive correction branch discussed in Section~\ref{sec:mft_parallel}.

\subsection{Exact Heteroscedastic Score-Noise Identity}
\label{app:der_heter_score_noise_id}

In this section, we establish the exact duality between the complex spectral score function and the cumulative noise under a fixed initial condition $\mcX_0$. This derivation formally links the score matching objective (SDE) with the noise prediction objective (DDPM) in the frequency domain.

\subsubsection{Exact Factorization of Conditional Quadrature Components}

According to the quadrature independence of the forward transition (Theorem~\ref{thm:freq_ortho}), the complex state at diffusion step $t$ can be expressed via the reparameterization trick as:
\begin{equation}
    \mcR_t 
    = \sqrt{\bar{\alpha}_t} \mcR_0 + \sqrt{1 - \bar{\alpha}_t} \bvareps_r , \quad
    \mcI_t 
    = \sqrt{\bar{\alpha}_t} \mcI_0 + \sqrt{1 - \bar{\alpha}_t} \bvareps_i
\end{equation}
where $\bvareps_r$ and $\bvareps_i$ represent the \textit{cumulative noise} injected from $t=0$ to $t$. They follow the spectral Gaussian distributions derived in Appendix~\ref{app:spec_dist_gauss}:
\begin{equation}
    \bvareps_r \sim \mcN(\mbzero, \bSigma_r) , \quad
    \bvareps_i \sim \mcN(\mbzero, \bSigma_i) .
\end{equation}

Consequently, the conditional distribution of the real component follows a multivariate Gaussian:
\begin{equation}
    \mcR_t | \mcR_0 \sim \mcN(\bmu^{\mcR}_t, \bSigma^{\mcR}_t) , \quad \text{where} \quad
    \bmu^{\mcR}_t = \sqrt{\bar{\alpha}_t} \mcR_0 , \quad
    \bSigma^{\mcR}_t = (1 - \bar{\alpha}_t) \bSigma_r .
\end{equation}

The corresponding conditional density function is: 
\begin{equation}
    p_{t|0}(\mcR)
    = \frac{1}{\sqrt{(2\pi)^L \det(\bSigma^{\mcR}_t)}}
    \exp \left(
        - \frac{1}{2} 
        ( \mcR_t - \bmu^{\mcR}_t )^{\top}
        ( \bSigma^{\mcR}_t )^{-1}
        ( \mcR_t - \bmu^{\mcR}_t )
    \right)
\end{equation}

\subsubsection{Derivation of Quadrature Score Functions}
Taking the logarithm of the conditional density for the real component:
\begin{equation}
\begin{aligned}
    \log p_{t|0}(\mcR) 
    &= - \frac{1}{2} 
        ( \mcR_t - \bmu^{\mcR}_t )^{\top}
        ( \bSigma^{\mcR}_t )^{-1}
        ( \mcR_t - \bmu^{\mcR}_t )
    + \text{Const}
    \\
    &= - \frac{1}{2(1 - \bar{\alpha}_t)}
        ( \mcR_t - \sqrt{\bar{\alpha}_t} \mcR_0 )^{\top}
        \bSigma_r^{-1}
        ( \mcR_t - \sqrt{\bar{\alpha}_t} \mcR_0 )
    + \text{Const} .
\end{aligned}
\end{equation}

To find the score function $\nabla_{\mcR} \log p_{t|0}(\mcR)$, we utilize the matrix calculus identity $\nabla_{\vecx} (\vecx - \mathbf{a})^{\top} \mathbf{A} (\vecx - \mathbf{a}) = (\mathbf{A} + \mathbf{A}^{\top}) (\vecx - \mathbf{a})$. Noting that the spectral covariance matrix $\bSigma_r$ (and thus its inverse) is symmetric:
\begin{equation}
\begin{aligned}
    \nabla_{\mcR} \log p_{t|0}(\mcR)
    &= - \frac{1}{2(1 - \bar{\alpha}_t)}
    \left[ \bSigma_r^{-1} + (\bSigma_r^{-1})^{\top} \right]
    ( \mcR_t - \sqrt{\bar{\alpha}_t} \mcR_0 )
    \\
    &= - \frac{1}{1 - \bar{\alpha}_t} \bSigma_r^{-1}
    \underbrace{( \mcR_t - \sqrt{\bar{\alpha}_t} \mcR_0 )}_{\sqrt{1 - \bar{\alpha}_t} \bvareps_r}
    \\
    &= - \frac{\bSigma_r^{-1} \bvareps_r}{\sqrt{1 - \bar{\alpha}_t}} .
\end{aligned}
\end{equation}

By symmetry, the score function for the imaginary component is derived analogously: \begin{equation} 
    \nabla_{\mcI} \log p_{t|0}(\mcI) 
    = - \frac{\bSigma_i^{-1} \bvareps_i}{\sqrt{1 - \bar{\alpha}_t}} . 
\end{equation}

\subsubsection{Complex Score Merger via Wirtinger Calculus}
We now unify these real-valued results into the complex domain. The \textbf{Wirtinger derivative} (conjugate gradient) is defined as:
$\nabla_{\bar{\mcX}} \coloneqq \frac{1}{2} \left( \nabla_{\mcR} + j \nabla_{\mcI} \right)$.
Since the joint conditional density factorizes as $p_{t|0}(\mcX) = p_{t|0}(\mcR) p_{t|0}(\mcI)$, the log-density exhibits an exact additive structure: 
\begin{equation}
    \log p_{t|0}(\mcX) = \log p_{t|0}(\mcR) + \log p_{t|0}(\mcI) .    
\end{equation}

Substituting the quadrature scores into the Wirtinger operator: 
\begin{equation} 
\begin{aligned} 
    \nabla_{\bar{\mcX}} \log p_{t|0}(\mcX) 
    &= \frac{1}{2} \left( \nabla_{\mcR} \log p_{t|0}(\mcX) + j \nabla_{\mcI} \log p_{t|0}(\mcX) \right) 
    \\ 
    &= \frac{1}{2} \left( \nabla_{\mcR} \log p_{t|0}(\mcR) + j \nabla_{\mcI} \log p_{t|0}(\mcI) \right) 
    \\ 
    &= \frac{1}{2} \left( - \frac{\bSigma_r^{-1} \bvareps_r}{\sqrt{1 - \bar{\alpha}_t}} - j \frac{\bSigma_i^{-1} \bvareps_i}{\sqrt{1 - \bar{\alpha}_t}} \right) 
    \\ 
    &= - \frac{ \bSigma_r^{-1} \bvareps_r + j \bSigma_i^{-1} \bvareps_i }{2 \sqrt{1 - \bar{\alpha}_t}}. 
\end{aligned} 
\end{equation}

\paragraph{Conclusion.}
This derivation confirms the Heteroscedastic Score-Noise Identity. It reveals that the complex score function is not merely the noise scaled by a scalar (as in standard homoscedastic DDPMs), but is the precision-weighted noise vector:
\begin{equation}
    \nabla_{\bar{\mcX}} \log p_{t|0}(\mcX) 
    \propto - \left( \bSigma_r^{-1} \bvareps_r + j \bSigma_i^{-1} \bvareps_i \right).
\end{equation}
This explicitly justifies the use of the Mahalanobis distance (weighted by $\bSigma$) in our loss function to counterbalance the spectral heteroscedasticity $\bSigma^{-1}$, ensuring efficient and unbiased optimization.

\subsection{Marginal Score Approximation via Mean Field Theory}
\label{app:mft_marginal}

While the conditional score identity derived in Appendix~\ref{app:der_heter_score_noise_id} is exact for a fixed $\mcX_0$, the generative process requires the \textit{marginal score} $\nabla_{\bar{\mcX}} \log p_t(\mcX)$, which marginalizes over the unknown data distribution $p(\mcX_0)$. This section formalizes the mean-field approximation used by our parallel architecture.

\subsubsection{The Entanglement of Marginal Dynamics}
We start by expressing the marginal score as an expectation over the posterior. Using the identity $\nabla_{\bar{\mcX}} \log p_t(\mcX) = \frac{\nabla_{\bar{\mcX}} p_t(\mcX)}{p_t(\mcX)}$ and the integral definition $p_t(\mcX) = \int p_{t|0}(\mcX) p(\mcX_0) \diff \mcX_0$, we derive:
\begin{equation} 
\begin{aligned} 
    \nabla_{\bar{\mcX}} \log p_t(\mcX) 
    &= \frac{1}{p_t(\mcX)} \int \nabla_{\bar{\mcX}} p_{t|0}(\mcX) p(\mcX_0) \diff \mcX_0 
    \\ 
    &= \frac{1}{p_t(\mcX)} \int \left[ 
        p_{t|0}(\mcX) \nabla_{\bar{\mcX}} \log p_{t|0}(\mcX) 
    \right] p(\mcX_0) \diff \mcX_0 
    \\ 
    &= \int \left[ \nabla_{\bar{\mcX}} \log p_{t|0}(\mcX) \right] 
        \underbrace{ \frac{p_{t|0}(\mcX) p(\mcX_0)}{p_t(\mcX)} }_{p(\mcX_0 | \mcX_t) \text{ (Posterior)}} \diff \mcX_0 \\ 
    &= \expt_{p(\mcX_0 | \mcX_t)} \left[ \nabla_{\bar{\mcX}} \log p_{t|0}(\mcX) \right] 
\end{aligned} 
\end{equation}

Substituting the conditional score decomposition from Eq.~\eqref{eq:nabla_separation} (Appendix~\ref{app:reverse_SDE}): 
\begin{equation} 
    \nabla_{\bar{\mcX}} \log p_t(\mcX) 
    = \frac{1}{2} \int \left[ 
        \nabla_{\mcR} \log p_{t|0}(\mcR) + j \nabla_{\mcI} \log p_{t|0}(\mcI) 
    \right] p(\mcX_0 | \mcX_t) \diff \mcX_0 
\label{eq:entangled_marginal_score} 
\end{equation}

\textbf{The Coupling Problem:} In real-world time series, the real and imaginary components of the data prior $p(\mcX_0)$ are strongly coupled (\eg phase coherence). Consequently, the posterior $p(\mcX_0 | \mcX_t)$ does not factorize, meaning the expectation in Eq.~\eqref{eq:entangled_marginal_score} entangles the real and imaginary gradients. The marginal score for the real part depends on the imaginary part, and vice versa.

\subsubsection{MFT Variational Proxy}
To bypass this representational rigidity and enable efficient parallel modeling, we invoke Mean Field Theory (MFT). We approximate the complex joint distribution using a factorized variational proxy. Specifically, within this proxy we replace the coupled data prior by the factorized form $p(\mcX_0) \approx p(\mcR_0) p(\mcI_0)$.

By the exact conditional factorization $p(\mcX_t | \mcX_0) = p(\mcR_t | \mcR_0) p(\mcI_t | \mcI_0)$ (Lemma~\ref{lem:cond_factor_traj}), the joint distribution approximately factorizes:
\begin{equation}
    p(\mcX_t, \mcX_0) 
    \approx \left[ p(\mcR_t | \mcR_0) p(\mcR_0) \right] \cdot \left[ p(\mcI_t | \mcI_0) p(\mcI_0) \right].
\end{equation}

Under this proxy, the marginal density and the posterior also approximately factorize: 
\begin{gather} 
    p_t(\mcX) 
    = \int p_{t|0}(\mcX) p(\mcX_0) \diff \mcX_0
    \approx \int p_{t|0}(\mcR) p(\mcR_0) \diff \mcR_0 \cdot \int p_{t|0}(\mcI) p(\mcI_0) \diff \mcI_0
    = p_t(\mcR) p_t(\mcI)
    \\
    p_{0|t}(\mcX) 
    = \frac{p_{t|0}(\mcX) p(\mcX_0)}{p_t(\mcX)}
    \approx \frac{p_{t|0}(\mcR) p(\mcR_0)}{p_t(\mcR)} \cdot \frac{p_{t|0}(\mcI) p(\mcI_0)}{p_t(\mcI)}
    = p_{0|t}(\mcR) p_{0|t}(\mcI). 
\end{gather}

Applying this factorized posterior to the score equation:
\begin{equation}
\begin{aligned}
    \nabla_{\bar{\mcX}} \log p_t(\mcX)
    &\approx \frac{1}{2} \iint \left[
    \nabla_{\mcR} \log p_{t|0}(\mcR) + j \nabla_{\mcI} \log p_{t|0}(\mcI)
    \right]
    p_{0|t}(\mcR) p_{0|t}(\mcI) \diff \mcR_0 \diff \mcI_0
    \\
    &= \frac{1}{2} \left(
        \underbrace{ \int [\nabla_{\mcR} \log p_{t|0}(\mcR)] p_{0|t}(\mcR) \diff \mcR_0 }_{\expt[\text{Real Score}]}
        \cdot \underbrace{ \int p_{0|t}(\mcI) \diff \mcI_0 }_{1}
    \right)
    + \frac{j}{2} \left(
        \underbrace{ \int [\nabla_{\mcI} \log p_{t|0}(\mcI)] p_{0|t}(\mcI) \diff \mcI_0 }_{\expt[\text{Imag Score}]}
        \cdot \underbrace{ \int p_{0|t}(\mcR) \diff \mcR_0 }_{1}
    \right)
    \\
    &= \frac{1}{2} \left( \nabla_{\mcR} \log p_t(\mcR) + j \nabla_{\mcI} \log p_t(\mcI) \right)
\end{aligned}
\end{equation}
where we utilized the identity $\int [\nabla_{\mathbf{x}} \log p(\mathbf{x}|\mathbf{z})] p(\mathbf{z}|\mathbf{x}) d\mathbf{z} = \nabla_{\mathbf{x}} \log p(\mathbf{x})$.

\subsubsection{Decoupling via Vanishing Cross-Gradients}
Equivalently, under the factorized marginal proxy $p_t(\mcX) \approx p_t(\mcR)p_t(\mcI)$, $p_t(\mcI)$ has no functional dependence on $\mcR$ and vice versa. The cross-gradients vanish, yielding the same decoupled marginal score:
\begin{equation}
    \nabla_{\bar{\mcX}} \log p_t(\mcX)
    \approx \frac{1}{2} \left(
        \nabla_{\mcR} \log p_t(\mcR) + j \nabla_{\mcI} \log p_t(\mcI)
    \right)
\end{equation}

This derivation justifies the use of independent loss functions for the real and imaginary branches in our parallel architecture. While this factorized objective enforces computational efficiency, the interactive correction branch introduced in PaCoDi compensates for the information lost during this MFT approximation.

\subsection{Loss Equivalence Across Diffusion Formulations}
\label{app:equiv_loss}

In this section, we provide a formal proof demonstrating that the continuous-time score-matching objective $\mcL^{\text{cont}}$ is equivalent to the heteroscedastic noise-prediction objective $\mcL^{\text{disc}}$ used in the discrete-time formulation of PaCoDi.

\subsubsection{Theoretical Expansion}
We start from the continuous score-matching loss for the real component, where the score residual is weighted by the spectral covariance $\bSigma_r$:
\begin{equation}
    \mcL_{\text{real}}^{\text{cont}}(\theta_r)
    = \expt_{\mcX_0, \mcE, t} \left[
        \lambda(t) \left \Vert 
            s_{\theta_r}(\mcR, h(\mcI), t) - \nabla_{\mcR} \log p_{t|0}(\mcR) 
        \right \Vert_{\bSigma_r}^2
    \right] .
\end{equation}

Expanding the spectral-covariance-weighted norm $\Vert \mathbf{v} \Vert_{\bSigma}^2 = \mathbf{v}^\top \bSigma \mathbf{v}$ into its quadratic form:
\begin{equation}
    \mcL_{\text{real}}^{\text{cont}}(\theta_r) 
    = \expt_{\mcX_0, \mcE, t} \left[
        \lambda(t)
        \left( s_{\theta_r} - \nabla_{\mcR} \log p_{t|0} \right)^{\top}
        \bSigma_r
        \left( s_{\theta_r} - \nabla_{\mcR} \log p_{t|0} \right)
    \right],
\end{equation}
where we omit arguments for brevity.

\subsubsection{Substitution of the Score-Noise Identity}
Using the Heteroscedastic Score-Noise Identity derived in Appendix~\ref{app:der_heter_score_noise_id}, the target score and the corresponding noise-parameterized score estimator are
\begin{equation} 
    s_{\theta_r} 
    = - \frac{\bSigma_r^{-1} \bvareps_{\theta_r}}{\sqrt{1 - \bar{\alpha}_t}} , \quad 
    \nabla_{\mcR} \log p_{t|0} 
    = - \frac{\bSigma_r^{-1} \bvareps_r}{\sqrt{1 - \bar{\alpha}_t}} . 
\end{equation}

Plugging these into the loss function: 
\begin{equation} 
\begin{aligned} 
    \mcL_{\text{real}}^{\text{cont}}(\theta_r) 
    &= \expt_{\mcX_0, \mcE, t} \left[ 
        \lambda(t) 
        \left( 
            - \frac{ \bSigma_r^{-1} (\bvareps_{\theta_r} - \bvareps_r) }{\sqrt{1 - \bar{\alpha}_t}} 
        \right)^{\top} 
        \bSigma_r 
        \left( 
            - \frac{ \bSigma_r^{-1} (\bvareps_{\theta_r} - \bvareps_r) }{\sqrt{1 - \bar{\alpha}_t}}
        \right) 
    \right] 
    \\ 
    &= \expt_{\mcX_0, \mcE, t} \left[ 
        \frac{\lambda(t)}{1 - \bar{\alpha}_t} 
        (\bvareps_{\theta_r} - \bvareps_r)^{\top} 
        \underbrace{ (\bSigma_r^{-1})^{\top} \bSigma_r \bSigma_r^{-1} }_{\text{Matrix Simplification}} 
        (\bvareps_{\theta_r} - \bvareps_r) 
    \right] . 
\end{aligned} 
\end{equation}

\subsubsection{Matrix and Coefficient Simplification}
Since the spectral covariance matrix $\bSigma_r$ is symmetric and positive definite, $(\bSigma_r^{-1})^{\top} = \bSigma_r^{-1}$. The central matrix term simplifies algebraically:
\begin{equation}
    (\bSigma_r^{-1})^{\top} \bSigma_r \bSigma_r^{-1}
    = \bSigma_r^{-1} (\bSigma_r \bSigma_r^{-1})
    = \bSigma_r^{-1} \mbI
    = \bSigma_r^{-1}.
\end{equation}

Furthermore, the weights used in the main text satisfy
\begin{equation}
    \lambda(t)=\frac{1-\alpha_t}{2\alpha_t}, \quad
    \lambda_t=\frac{1-\alpha_t}{2\alpha_t(1-\bar{\alpha}_t)}, \quad
    \frac{\lambda(t)}{1 - \bar{\alpha}_t} = \lambda_t .
\end{equation}
% In standard DDPM, $\lambda_t = 1$, which implies $\lambda(t) = 1 - \bar{\alpha}_t$ is the implicit weighting for score matching.

Substituting the simplified matrix and coefficient back, we arrive at: 
\begin{equation} 
\begin{aligned} 
    \mcL_{\text{real}}^{\text{cont}}(\theta_r) 
    &= \expt_{\mcX_0, \mcE, t} \left[ 
        \lambda_t 
        (\bvareps_{\theta_r} - \bvareps_r)^{\top} 
        \bSigma_r^{-1} 
        (\bvareps_{\theta_r} - \bvareps_r) 
    \right] 
    \\ 
    &= \expt_{\mcX_0, \mcE, t} \left[ 
        \lambda_t \left \Vert 
            \bvareps_{\theta_r}(\mcR, h(\mcI), t) - \bvareps_r 
        \right \Vert_{\bSigma_r^{-1}}^2 
    \right] 
    \\ 
    &= \mcL_{\text{real}}^{\text{disc}}(\theta_r) . 
\end{aligned} 
\end{equation}

The derivation for the imaginary component $\mcL_{\text{imag}}^{\text{cont}}(\theta_i)$ follows an identical path using $\bSigma_i$. This proves that minimizing the continuous-time score matching objective with spectral-covariance-weighted score residuals is mathematically equivalent to minimizing the discrete-time noise prediction error weighted by the spectral precision matrix.

\subsubsection{Temporal-Frequency Loss Equivalence}
The frequency-domain loss is also equivalent to the temporal-domain diffusion loss. Let $\delta\beps=\beps_{\theta}-\beps$ denote the temporal noise-prediction residual and let $\delta\mcE=\mbU\delta\beps$ be its DFT-domain residual. Since the normalized DFT matrix is unitary, $\mbU^{\herm}\mbU=\mbI$, Parseval's identity gives
\begin{equation}
    \Vert \delta\beps \Vert_2^2
    = \Vert \mbU\delta\beps \Vert_2^2
    = \Vert \delta\mcE \Vert_2^2 .
\end{equation}

On the independent Hermitian coordinate chart, the same spectral residual decomposes as $\delta\mcE=\delta\bvareps_r+j\delta\bvareps_i$ with augmented covariance $\operatorname{blkdiag}(\bSigma_r,\bSigma_i)$. The induced spectral precision recovers the temporal Euclidean quadratic form:
\begin{equation}
    \Vert \delta\beps \Vert_2^2
    =
    \delta\bvareps_r^{\top}\bSigma_r^{-1}\delta\bvareps_r
    +
    \delta\bvareps_i^{\top}\bSigma_i^{-1}\delta\bvareps_i .
\end{equation}
Thus, the temporal diffusion loss and the frequency diffusion loss are equivalent under the unitary DFT, while the spectral precision form makes the heteroscedastic quadrature geometry explicit.

\newpage
\section{Mathematical Lemmas}
\label{app:math_lemmas}

This appendix collects the elementary probabilistic and Fourier identities used in Appendix~\ref{app:spec_dist_gauss}.
We state them separately to keep the covariance derivations focused on the spectral structure.

\begin{tcolorbox}[tcbset]
\begin{lemma}
\label{lem:exp_prod_Gauss_noise}
    \textbf{(Product Moment of Independent Gaussian Noise).}
    Let $\{\beps_{\tau}\}_{\tau \in \mathbb{Z}}$ be independent Gaussian random variables with $\beps_{\tau} \sim \mcN(\mu,\sigma^2)$.
    For any two indices $n$ and $m$,
    \begin{equation}
        \expt[\beps_n \beps_m]
        = \mu^2+\sigma^2\delta_{n,m},
    \end{equation}
    where $\delta_{n,m}$ is the Kronecker delta.
    In particular, for the centered noise used in this paper, $\mu=0$ and $\expt[\beps_n\beps_m]=\sigma^2\delta_{n,m}$.
\end{lemma}
\end{tcolorbox}

\begin{proof}
If $n=m$, then
\begin{equation}
    \expt[\beps_n^2]
    = \operatorname{Var}(\beps_n)+(\expt[\beps_n])^2
    = \sigma^2+\mu^2 .
\end{equation}
If $n\neq m$, independence gives
\begin{equation}
    \expt[\beps_n\beps_m]
    = \expt[\beps_n]\expt[\beps_m]
    = \mu^2 .
\end{equation}
Combining the two cases yields $\expt[\beps_n \beps_m]=\mu^2+\sigma^2\delta_{n,m}$.
\end{proof}

\begin{tcolorbox}[tcbset]
\begin{lemma}
\label{lem:trig_sum}
    \textbf{(Discrete Orthogonality of Fourier Modes).}
    Let $L \in \mathbb{Z}_{>0}$ and $A \in \mathbb{Z}$.
    Define the modular indicator
    \begin{equation}
        \delta_L(A)=
        \begin{cases}
            1, & A \equiv 0 \pmod L,\\
            0, & A \not\equiv 0 \pmod L .
        \end{cases}
    \end{equation}
    Then the following full-period sums hold:
    \begin{equation}
        \sum_{n=0}^{L-1} \cos\left(\frac{2\pi A n}{L}\right)
        = L \delta_L(A) , \quad
        \sum_{n=0}^{L-1} \sin\left(\frac{2\pi A n}{L}\right)
        =0 .
    \end{equation}
\end{lemma}
\end{tcolorbox}

In Appendix~\ref{app:spec_dist_gauss}, $A$ is instantiated as $k-l$ or $k+l$.

\begin{proof}
Consider the complex exponential sum
\begin{equation}
    S(A)=\sum_{n=0}^{L-1}\exp\left(j\frac{2\pi A n}{L}\right).
\end{equation}
If $A\equiv 0 \pmod L$, each term equals one, and hence $S(A)=L$.
If $A\not\equiv 0 \pmod L$, let $\omega=\exp(j2\pi A/L)$.
Then $\omega \neq 1$ and $\omega^L=\exp(j2\pi A)=1$, so the finite geometric series gives
\begin{equation}
    S(A)=\frac{1-\omega^L}{1-\omega}=0.
\end{equation}
Therefore,
\begin{equation}
    S(A)=L\delta_L(A).
\end{equation}
Taking the real and imaginary parts of $S(A)$ proves the cosine and sine identities.
\end{proof}

\clearpage
\onecolumn
\section{Experimental Results}
\label{app:results}

\begin{table}[H]
    \centering
    \begin{minipage}[t]{0.48\textwidth}
    \vspace{0pt}
    \centering
    \caption{
        \textbf{Unconditional Generation on ETTh1 Dataset}.
        \bst{Bold} and \subbst{underline} denote best and second best results.
    }
    \label{tab:uncond_etth1}
    \renewcommand{\arraystretch}{0.8}
    \setlength{\tabcolsep}{2.4pt}
    \begin{tabular}{c|c|ccccccc}
        \toprule
        \multicolumn{2}{c|}{\scaleb{Models}} & 
        \textbf{\scaleb{Cont.}} & \textbf{\scaleb{Disc.}} & \scaleb{FreqDiff} & \scaleb{Diff-TS} & 
        \scaleb{TimeVAE} & \scaleb{TimeGAN} & \scaleb{DDPM}    \\
        \midrule

        \multirow{5}{*}{\rotatebox{90}{\scalebox{0.9}{Context-FID}}}
        & \scalea{24} & 
        \bst{\scalea{$0.015_{\pm .001}$}} & \subbst{\scalea{$0.018_{\pm .002}$}} & \scalea{$0.032_{\pm .000}$} & \scalea{$0.140_{\pm .009}$} & 
        \scalea{$0.786_{\pm .075}$} & \scalea{$7.561_{\pm 1.163}$} & \scalea{$0.311_{\pm .321}$} \\
        & \scalea{64} & 
        \bst{\scalea{$0.053_{\pm .004}$}} & \subbst{\scalea{$0.077_{\pm .005}$}} & \scalea{$0.087_{\pm .000}$} & \scalea{$0.273_{\pm .033}$} & 
        \scalea{$0.856_{\pm .227}$} & \scalea{$9.266_{\pm .860}$} & \scalea{$0.438_{\pm .123}$} \\
        & \scalea{128} & 
        \bst{\scalea{$0.149_{\pm .006}$}} & \scalea{$0.175_{\pm .009}$} & \subbst{\scalea{$0.155_{\pm .000}$}} & \scalea{$1.144_{\pm .080}$} & 
        \scalea{$0.929_{\pm .046}$} & \scalea{$17.438_{\pm 2.269}$} & \scalea{$0.577_{\pm .085}$} \\
        & \scalea{256} & 
        \bst{\scalea{$0.591_{\pm .069}$}} & \bst{\scalea{$0.591_{\pm .069}$}} & \subbst{\scalea{$0.760_{\pm .000}$}} & \scalea{$2.152_{\pm .092}$} & 
        \scalea{$1.025_{\pm .080}$} & \scalea{$28.819_{\pm 2.444}$} & \scalea{$0.899_{\pm .237}$} \\
        \cmidrule(lr){2-9}
        & \scalea{\emph{Avg}} & 
        \bst{\scalea{0.202}} & \subbst{\scalea{0.215}} & \scalea{0.259} & \scalea{0.927} & 
        \scalea{0.899} & \scalea{15.771} & \scalea{0.556} \\
        \midrule
        
        \multirow{5}{*}{\rotatebox{90}{\scalebox{0.9}{Correlational}}}
        & \scalea{24} & 
        \subbst{\scalea{$0.033_{\pm .023}$}} & \scalea{$0.044_{\pm .014}$} & \scalea{$0.049_{\pm .000}$} & \scalea{$0.066_{\pm .016}$} & 
        \scalea{$0.103_{\pm .028}$} & \scalea{$1.482_{\pm .007}$} & \bst{\scalea{$0.032_{\pm .008}$}} \\
        & \scalea{64} & 
        \scalea{$0.045_{\pm .012}$} & \subbst{\scalea{$0.040_{\pm .005}$}} & \scalea{$0.055_{\pm .000}$} & \scalea{$0.069_{\pm .016}$} & 
        \scalea{$0.075_{\pm .020}$} & \scalea{$1.183_{\pm .020}$} & \bst{\scalea{$0.039_{\pm .016}$}} \\
        & \scalea{128} & 
        \bst{\scalea{$0.044_{\pm .006}$}} & \scalea{$0.053_{\pm .019}$} & \subbst{\scalea{$0.049_{\pm .000}$}} & \scalea{$0.089_{\pm .015}$} & 
        \scalea{$0.066_{\pm .031}$} & \scalea{$1.616_{\pm .007}$} & \scalea{$0.053_{\pm .034}$} \\
        & \scalea{256} & 
        \bst{\scalea{$0.040_{\pm .008}$}} & \scalea{$0.057_{\pm .016}$} & \subbst{\scalea{$0.044_{\pm .000}$}} & \scalea{$0.119_{\pm .009}$} & 
        \scalea{$0.045_{\pm .012}$} & \scalea{$1.480_{\pm .008}$} & \subbst{\scalea{$0.044_{\pm .019}$}} \\
        \cmidrule(lr){2-9}
        & \scalea{\emph{Avg}} & 
        \bst{\scalea{0.041}} & \scalea{0.049} & \scalea{0.049} & \scalea{0.086} & 
        \scalea{0.072} & \scalea{1.440} & \subbst{\scalea{0.042}} \\
        \midrule

        \multirow{5}{*}{\rotatebox{90}{\scalebox{0.9}{Discriminative}}}
        & \scalea{24} & 
        \subbst{\scalea{$0.008_{\pm .006}$}} & \bst{\scalea{$0.004_{\pm .004}$}} & \scalea{$0.021_{\pm .000}$} & \scalea{$0.073_{\pm .016}$} & 
        \scalea{$0.167_{\pm .138}$} & \scalea{$0.445_{\pm .087}$} & \scalea{$0.066_{\pm .142}$} \\
        & \scalea{64} & 
        \bst{\scalea{$0.022_{\pm .007}$}} & \subbst{\scalea{$0.032_{\pm .020}$}} & \scalea{$0.035_{\pm .000}$} & \scalea{$0.077_{\pm .007}$} & 
        \scalea{$0.167_{\pm .106}$} & \scalea{$0.456_{\pm .041}$} & \scalea{$0.203_{\pm .066}$} \\
        & \scalea{128} & 
        \bst{\scalea{$0.038_{\pm .035}$}} & \scalea{$0.094_{\pm .066}$} & \subbst{\scalea{$0.062_{\pm .000}$}} & \scalea{$0.171_{\pm .016}$} & 
        \scalea{$0.169_{\pm .141}$} & \scalea{$0.414_{\pm .158}$} & \scalea{$0.198_{\pm .382}$} \\
        & \scalea{256} & 
        \scalea{$0.227_{\pm .041}$} & \scalea{$0.216_{\pm .119}$} & \scalea{$0.237_{\pm .000}$} & \scalea{$0.222_{\pm .008}$} & 
        \subbst{\scalea{$0.199_{\pm .130}$}} & \scalea{$0.437_{\pm .073}$} & \bst{\scalea{$0.084_{\pm .210}$}} \\
        \cmidrule(lr){2-9}
        & \scalea{\emph{Avg}} & 
        \bst{\scalea{0.074}} & \subbst{\scalea{0.087}} & \scalea{0.089} & \scalea{0.136} & 
        \scalea{0.176} & \scalea{0.438} & \scalea{0.138} \\
        \midrule

        \multirow{5}{*}{\rotatebox{90}{\scalebox{0.9}{Predictive}}}
        & \scalea{24} & 
        \bst{\scalea{$0.100_{\pm .004}$}} & \scalea{$0.119_{\pm .004}$} & \scalea{$0.119_{\pm .000}$} & \subbst{\scalea{$0.118_{\pm .005}$}} & 
        \scalea{$0.128_{\pm .003}$} & \scalea{$0.199_{\pm .031}$} & \scalea{$0.119_{\pm .001}$} \\
        & \scalea{64} & 
        \bst{\scalea{$0.114_{\pm .005}$}} & \subbst{\scalea{$0.116_{\pm .008}$}} & \subbst{\scalea{$0.116_{\pm .000}$}} & \scalea{$0.117_{\pm .003}$} & 
        \scalea{$0.118_{\pm .002}$} & \scalea{$0.151_{\pm .007}$} & \scalea{$0.120_{\pm .011}$} \\
        & \scalea{128} & 
        \subbst{\scalea{$0.113_{\pm .009}$}} & \bst{\scalea{$0.108_{\pm .013}$}} & \subbst{\scalea{$0.113_{\pm .000}$}} & \scalea{$0.115_{\pm .007}$} & 
        \scalea{$0.119_{\pm .006}$} & \scalea{$0.144_{\pm .008}$} & \bst{\scalea{$0.108_{\pm .004}$}} \\
        & \scalea{256} & 
        \bst{\scalea{$0.102_{\pm .015}$}} & \subbst{\scalea{$0.105_{\pm .008}$}} & \scalea{$0.129_{\pm .000}$} & \scalea{$0.119_{\pm .009}$} & 
        \scalea{$0.112_{\pm .004}$} & \scalea{$0.149_{\pm .009}$} & \scalea{$0.115_{\pm .016}$} \\
        \cmidrule(lr){2-9}
        & \scalea{\emph{Avg}} & 
        \bst{\scalea{0.107}} & \subbst{\scalea{0.112}} & \scalea{0.119} & \scalea{0.117} & 
        \scalea{0.119} & \scalea{0.161} & \scalea{0.116} \\
        \bottomrule
    \end{tabular}
    \end{minipage}
    \hfill
    \begin{minipage}[t]{0.48\textwidth}
    \vspace{0pt}
    \centering
    \caption{
        \textbf{Unconditional Generation on Stocks Dataset}.
        \bst{Bold} and \subbst{underline} denote best and second best results.
    }
    \label{tab:uncond_stocks}
    \renewcommand{\arraystretch}{0.8}
    \setlength{\tabcolsep}{2.4pt}
    \begin{tabular}{c|c|ccccccc}
        \toprule
        \multicolumn{2}{c|}{\scaleb{Models}} & 
        \textbf{\scaleb{Cont.}} & \textbf{\scaleb{Disc.}} & \scaleb{FreqDiff} & \scaleb{Diff-TS} & 
        \scaleb{TimeVAE} & \scaleb{TimeGAN} & \scaleb{DDPM}    \\
        \midrule

        \multirow{5}{*}{\rotatebox{90}{\scalebox{0.9}{Context-FID}}}
        & \scalea{24} & 
        \subbst{\scalea{$0.022_{\pm .006}$}} & \bst{\scalea{$0.006_{\pm .001}$}} & \scalea{$0.038_{\pm .000}$} & \scalea{$0.187_{\pm .021}$} & 
        \scalea{$0.190_{\pm .043}$} & \scalea{$1.739_{\pm .558}$} & \scalea{$0.025_{\pm .004}$} \\
        & \scalea{64} & 
        \scalea{$0.096_{\pm .014}$} & \bst{\scalea{$0.032_{\pm .012}$}} & \subbst{\scalea{$0.090_{\pm .000}$}} & \scalea{$0.424_{\pm .026}$} & 
        \scalea{$0.371_{\pm .076}$} & \scalea{$2.143_{\pm .448}$} & \scalea{$0.093_{\pm .042}$} \\
        & \scalea{128} & 
        \scalea{$0.413_{\pm .045}$} & \subbst{\scalea{$0.338_{\pm .072}$}} & \bst{\scalea{$0.277_{\pm .000}$}} & \scalea{$0.599_{\pm .122}$} & 
        \scalea{$0.523_{\pm .076}$} & \scalea{$3.496_{\pm .548}$} & \scalea{$0.378_{\pm .140}$} \\
        & \scalea{256} & 
        \scalea{$0.556_{\pm .141}$} & \scalea{$0.459_{\pm .146}$} & \bst{\scalea{$0.268_{\pm .000}$}} & \scalea{$0.532_{\pm .071}$} & 
        \subbst{\scalea{$0.286_{\pm .066}$}} & \scalea{$4.578_{\pm 1.583}$} & \scalea{$2.046_{\pm .556}$} \\
        \cmidrule(lr){2-9}
        & \scalea{\emph{Avg}} & 
        \scalea{0.272} & \subbst{\scalea{0.209}} & \bst{\scalea{0.168}} & \scalea{0.436} & 
        \scalea{0.343} & \scalea{2.989} & \scalea{0.636} \\
        \midrule

        \multirow{5}{*}{\rotatebox{90}{\scalebox{0.9}{Correlational}}}
        & \scalea{24} & 
        \bst{\scalea{$0.004_{\pm .003}$}} & \subbst{\scalea{$0.006_{\pm .007}$}} & \subbst{\scalea{$0.006_{\pm .000}$}} & \scalea{$0.007_{\pm .006}$} & 
        \scalea{$0.107_{\pm .006}$} & \scalea{$0.223_{\pm .003}$} & \bst{\scalea{$0.004_{\pm .004}$}} \\
        & \scalea{64} & 
        \scalea{$0.009_{\pm .005}$} & \scalea{$0.011_{\pm .007}$} & \subbst{\scalea{$0.006_{\pm .000}$}} & \scalea{$0.014_{\pm .004}$} & 
        \scalea{$0.093_{\pm .006}$} & \scalea{$0.218_{\pm .005}$} & \bst{\scalea{$0.003_{\pm .004}$}} \\
        & \scalea{128} & 
        \scalea{$0.014_{\pm .004}$} & \bst{\scalea{$0.007_{\pm .003}$}} & \subbst{\scalea{$0.008_{\pm .000}$}} & \scalea{$0.020_{\pm .006}$} & 
        \scalea{$0.087_{\pm .001}$} & \scalea{$0.217_{\pm .005}$} & \scalea{$0.015_{\pm .004}$} \\
        & \scalea{256} & 
        \scalea{$0.018_{\pm .004}$} & \subbst{\scalea{$0.008_{\pm .008}$}} & \bst{\scalea{$0.007_{\pm .000}$}} & \scalea{$0.009_{\pm .005}$} & 
        \scalea{$0.062_{\pm .007}$} & \scalea{$0.760_{\pm .004}$} & \scalea{$0.021_{\pm .015}$} \\
        \cmidrule(lr){2-9}
        & \scalea{\emph{Avg}} & 
        \scalea{0.011} & \subbst{\scalea{0.008}} & \bst{\scalea{0.007}} & \scalea{0.013} & 
        \scalea{0.087} & \scalea{0.355} & \scalea{0.011} \\
        \midrule

        \multirow{5}{*}{\rotatebox{90}{\scalebox{0.9}{Discriminative}}}
        & \scalea{24} & 
        \bst{\scalea{$0.014_{\pm .011}$}} & \subbst{\scalea{$0.015_{\pm .021}$}} & \scalea{$0.024_{\pm .000}$} & \scalea{$0.086_{\pm .007}$} & 
        \scalea{$0.100_{\pm .075}$} & \scalea{$0.459_{\pm .021}$} & \scalea{$0.108_{\pm .179}$} \\
        & \scalea{64} & 
        \bst{\scalea{$0.016_{\pm .013}$}} & \scalea{$0.046_{\pm .010}$} & \subbst{\scalea{$0.023_{\pm .000}$}} & \scalea{$0.081_{\pm .029}$} & 
        \scalea{$0.149_{\pm .099}$} & \scalea{$0.315_{\pm .059}$} & \scalea{$0.222_{\pm .003}$} \\
        & \scalea{128} & 
        \subbst{\scalea{$0.061_{\pm .017}$}} & \scalea{$0.118_{\pm .051}$} & \bst{\scalea{$0.047_{\pm .000}$}} & \scalea{$0.122_{\pm .006}$} & 
        \scalea{$0.127_{\pm .144}$} & \scalea{$0.196_{\pm .141}$} & \scalea{$0.204_{\pm .134}$} \\
        & \scalea{256} & 
        \scalea{$0.173_{\pm .008}$} & \bst{\scalea{$0.064_{\pm .010}$}} & \subbst{\scalea{$0.083_{\pm .000}$}} & \scalea{$0.134_{\pm .020}$} & 
        \scalea{$0.139_{\pm .159}$} & \scalea{$0.397_{\pm .151}$} & \scalea{$0.244_{\pm .274}$} \\
        \cmidrule(lr){2-9}
        & \scalea{\emph{Avg}} & 
        \scalea{0.066} & \subbst{\scalea{0.061}} & \bst{\scalea{0.044}} & \scalea{0.106} & 
        \scalea{0.129} & \scalea{0.342} & \scalea{0.195} \\
        \midrule

        \multirow{5}{*}{\rotatebox{90}{\scalebox{0.9}{Predictive}}}
        & \scalea{24} & 
        \bst{\scalea{$0.037_{\pm .000}$}} & \bst{\scalea{$0.037_{\pm .000}$}} & \bst{\scalea{$0.037_{\pm .000}$}} & \bst{\scalea{$0.037_{\pm .000}$}} & 
        \scalea{$0.040_{\pm .000}$} & \scalea{$0.061_{\pm .001}$} & \subbst{\scalea{$0.039_{\pm .001}$}} \\
        & \scalea{64} & 
        \subbst{\scalea{$0.037_{\pm .000}$}} & \bst{\scalea{$0.036_{\pm .000}$}} & \subbst{\scalea{$0.037_{\pm .000}$}} & \subbst{\scalea{$0.037_{\pm .000}$}} & 
        \scalea{$0.038_{\pm .000}$} & \scalea{$0.043_{\pm .001}$} & \scalea{$0.042_{\pm .001}$} \\
        & \scalea{128} & 
        \subbst{\scalea{$0.037_{\pm .000}$}} & \scalea{$0.041_{\pm .000}$} & \bst{\scalea{$0.036_{\pm .000}$}} & \bst{\scalea{$0.036_{\pm .000}$}} & 
        \subbst{\scalea{$0.037_{\pm .000}$}} & \scalea{$0.050_{\pm .002}$} & \scalea{$0.041_{\pm .001}$} \\
        & \scalea{256} & 
        \scalea{$0.037_{\pm .000}$} & \scalea{$0.038_{\pm .000}$} & \bst{\scalea{$0.035_{\pm .000}$}} & \subbst{\scalea{$0.036_{\pm .000}$}} & 
        \bst{\scalea{$0.035_{\pm .002}$}} & \scalea{$0.099_{\pm .035}$} & \scalea{$0.044_{\pm .000}$} \\
        \cmidrule(lr){2-9}
        & \scalea{\emph{Avg}} & 
        \subbst{\scalea{0.037}} & \scalea{0.038} & \bst{\scalea{0.036}} & \subbst{\scalea{0.037}} & 
        \scalea{0.038} & \scalea{0.063} & \scalea{0.042} \\
        \bottomrule
    \end{tabular}
    \end{minipage}
\end{table}

\begin{table}[H]
    \centering
    \begin{minipage}[t]{0.48\textwidth}
    \vspace{0pt}
    \centering
    \caption{
        \textbf{Unconditional Generation on Sines Dataset}.
        \bst{Bold} and \subbst{underline} denote best and second best results.
    }
    \label{tab:uncond_sines}
    \renewcommand{\arraystretch}{0.8}
    \setlength{\tabcolsep}{2.4pt}
    \begin{tabular}{c|c|ccccccc}
        \toprule
        \multicolumn{2}{c|}{\scaleb{Models}} & 
        \textbf{\scaleb{Cont.}} & \textbf{\scaleb{Disc.}} & \scaleb{FreqDiff} & \scaleb{Diff-TS} & 
        \scaleb{TimeVAE} & \scaleb{TimeGAN} & \scaleb{DDPM}    \\
        \midrule

        \multirow{5}{*}{\rotatebox{90}{\scalebox{0.9}{Context-FID}}}
        & \scalea{24} & 
        \subbst{\scalea{$0.007_{\pm .001}$}} & \bst{\scalea{$0.003_{\pm .001}$}} & \scalea{$0.013_{\pm .000}$} & \scalea{$0.009_{\pm .002}$} & 
        \scalea{$0.291_{\pm .004}$} & \scalea{$0.063_{\pm .005}$} & \scalea{$0.011_{\pm .004}$} \\
        & \scalea{64} & 
        \bst{\scalea{$0.012_{\pm .002}$}} & \scalea{$0.018_{\pm .003}$} & \subbst{\scalea{$0.017_{\pm .000}$}} & \scalea{$0.029_{\pm .004}$} & 
        \scalea{$1.263_{\pm .069}$} & \scalea{$1.392_{\pm .065}$} & \scalea{$0.038_{\pm .010}$} \\
        & \scalea{128} & 
        \subbst{\scalea{$0.043_{\pm .002}$}} & \scalea{$0.062_{\pm .005}$} & \bst{\scalea{$0.038_{\pm .000}$}} & \scalea{$0.113_{\pm .025}$} & 
        \scalea{$0.810_{\pm .120}$} & \scalea{$15.224_{\pm .698}$} & \scalea{$0.070_{\pm .023}$} \\
        & \scalea{256} & 
        \scalea{$0.213_{\pm .022}$} & \scalea{$0.317_{\pm .014}$} & \subbst{\scalea{$0.201_{\pm .000}$}} & \scalea{$0.442_{\pm .030}$} & 
        \scalea{$0.556_{\pm .078}$} & \scalea{$26.385_{\pm 1.266}$} & \bst{\scalea{$0.163_{\pm .046}$}} \\
        \cmidrule(lr){2-9}
        & \scalea{\emph{Avg}} & 
        \subbst{\scalea{0.069}} & \scalea{0.100} & \bst{\scalea{0.067}} & \scalea{0.148} & 
        \scalea{0.730} & \scalea{10.766} & \scalea{0.071} \\
        \midrule

        \multirow{5}{*}{\rotatebox{90}{\scalebox{0.9}{Correlational}}}
        & \scalea{24} & 
        \scalea{$0.026_{\pm .011}$} & \bst{\scalea{$0.019_{\pm .006}$}} & \scalea{$0.027_{\pm .000}$} & \scalea{$0.035_{\pm .012}$} & 
        \scalea{$0.035_{\pm .008}$} & \scalea{$0.041_{\pm .009}$} & \subbst{\scalea{$0.021_{\pm .015}$}} \\
        & \scalea{64} & 
        \subbst{\scalea{$0.023_{\pm .007}$}} & \scalea{$0.030_{\pm .010}$} & \bst{\scalea{$0.022_{\pm .000}$}} & \subbst{\scalea{$0.023_{\pm .009}$}} & 
        \scalea{$0.165_{\pm .011}$} & \scalea{$0.137_{\pm .010}$} & \scalea{$0.030_{\pm .004}$} \\
        & \scalea{128} & 
        \bst{\scalea{$0.017_{\pm .004}$}} & \scalea{$0.021_{\pm .003}$} & \scalea{$0.025_{\pm .000}$} & \scalea{$0.040_{\pm .009}$} & 
        \scalea{$0.093_{\pm .009}$} & \scalea{$0.858_{\pm .009}$} & \subbst{\scalea{$0.019_{\pm .004}$}} \\
        & \scalea{256} & 
        \bst{\scalea{$0.011_{\pm .005}$}} & \scalea{$0.015_{\pm .004}$} & \scalea{$0.013_{\pm .000}$} & \scalea{$0.021_{\pm .005}$} & 
        \scalea{$0.050_{\pm .007}$} & \scalea{$0.918_{\pm .005}$} & \subbst{\scalea{$0.012_{\pm .007}$}} \\
        \cmidrule(lr){2-9}
        & \scalea{\emph{Avg}} & 
        \bst{\scalea{0.019}} & \subbst{\scalea{0.021}} & \scalea{0.022} & \scalea{0.030} & 
        \scalea{0.086} & \scalea{0.489} & \subbst{\scalea{0.021}} \\
        \midrule

        \multirow{5}{*}{\rotatebox{90}{\scalebox{0.9}{Discriminative}}}
        & \scalea{24} & 
        \scalea{$0.008_{\pm .004}$} & \subbst{\scalea{$0.007_{\pm .008}$}} & \scalea{$0.022_{\pm .000}$} & \bst{\scalea{$0.006_{\pm .006}$}} & 
        \scalea{$0.021_{\pm .020}$} & \scalea{$0.020_{\pm .013}$} & \scalea{$0.012_{\pm .009}$} \\
        & \scalea{64} & 
        \bst{\scalea{$0.008_{\pm .005}$}} & \subbst{\scalea{$0.010_{\pm .003}$}} & \scalea{$0.019_{\pm .000}$} & \scalea{$0.023_{\pm .019}$} & 
        \scalea{$0.101_{\pm .202}$} & \scalea{$0.162_{\pm .064}$} & \scalea{$0.019_{\pm .008}$} \\
        & \scalea{128} & 
        \bst{\scalea{$0.013_{\pm .016}$}} & \scalea{$0.022_{\pm .007}$} & \scalea{$0.027_{\pm .000}$} & \scalea{$0.116_{\pm .019}$} & 
        \subbst{\scalea{$0.019_{\pm .009}$}} & \scalea{$0.477_{\pm .019}$} & \scalea{$0.027_{\pm .041}$} \\
        & \scalea{256} & 
        \scalea{$0.027_{\pm .006}$} & \subbst{\scalea{$0.023_{\pm .019}$}} & \scalea{$0.064_{\pm .000}$} & \scalea{$0.163_{\pm .105}$} & 
        \scalea{$0.130_{\pm .214}$} & \scalea{$0.378_{\pm .263}$} & \bst{\scalea{$0.011_{\pm .013}$}} \\
        \cmidrule(lr){2-9}
        & \scalea{\emph{Avg}} & 
        \bst{\scalea{0.014}} & \subbst{\scalea{0.016}} & \scalea{0.033} & \scalea{0.077} & 
        \scalea{0.068} & \scalea{0.259} & \scalea{0.017} \\
        \midrule

        \multirow{5}{*}{\rotatebox{90}{\scalebox{0.9}{Predictive}}}
        & \scalea{24} & 
        \bst{\scalea{$0.093_{\pm .000}$}} & \bst{\scalea{$0.093_{\pm .000}$}} & \bst{\scalea{$0.093_{\pm .000}$}} & \bst{\scalea{$0.093_{\pm .001}$}} & 
        \scalea{$0.227_{\pm .000}$} & \subbst{\scalea{$0.095_{\pm .001}$}} & \bst{\scalea{$0.093_{\pm .000}$}} \\
        & \scalea{64} & 
        \bst{\scalea{$0.188_{\pm .003}$}} & \subbst{\scalea{$0.191_{\pm .002}$}} & \scalea{$0.192_{\pm .000}$} & \scalea{$0.193_{\pm .001}$} & 
        \scalea{$0.213_{\pm .001}$} & \scalea{$0.217_{\pm .005}$} & \bst{\scalea{$0.188_{\pm .007}$}} \\
        & \scalea{128} & 
        \subbst{\scalea{$0.252_{\pm .003}$}} & \subbst{\scalea{$0.252_{\pm .003}$}} & \bst{\scalea{$0.250_{\pm .000}$}} & \scalea{$0.257_{\pm .003}$} & 
        \scalea{$0.264_{\pm .002}$} & \scalea{$0.279_{\pm .008}$} & \scalea{$0.255_{\pm .001}$} \\
        & \scalea{256} & 
        \bst{\scalea{$0.286_{\pm .003}$}} & \subbst{\scalea{$0.287_{\pm .003}$}} & \bst{\scalea{$0.286_{\pm .000}$}} & \scalea{$0.288_{\pm .003}$} & 
        \scalea{$0.292_{\pm .003}$} & \scalea{$0.328_{\pm .014}$} & \scalea{$0.288_{\pm .003}$} \\
        \cmidrule(lr){2-9}
        & \scalea{\emph{Avg}} & 
        \bst{\scalea{0.205}} & \subbst{\scalea{0.206}} & \bst{\scalea{0.205}} & \scalea{0.208} & 
        \scalea{0.249} & \scalea{0.230} & \subbst{\scalea{0.206}} \\
        \bottomrule
    \end{tabular}
    \end{minipage}
    \hfill
    \begin{minipage}[t]{0.48\textwidth}
    \vspace{0pt}
    \centering
    \caption{
        \textbf{Unconditional Generation on Air Quality Dataset}.
        \bst{Bold} and \subbst{underline} denote best and second best results.
    }
    \label{tab:uncond_air}
    \renewcommand{\arraystretch}{0.8}
    \setlength{\tabcolsep}{2.4pt}
    \begin{tabular}{c|c|ccccccc}
        \toprule
        \multicolumn{2}{c|}{\scaleb{Models}} & 
        \scaleb{Cont.}   & \scaleb{Disc.}   & \scaleb{FreqDiff} & \scaleb{Diff-TS} & 
        \scaleb{TimeVAE} & \scaleb{TimeGAN} & \scaleb{DDPM}    \\
        \midrule

        \multirow{5}{*}{\rotatebox{90}{\scalebox{0.9}{Context-FID}}}
        & \scalea{24} & 
        \subbst{\scalea{$0.026_{\pm .001}$}} & \bst{\scalea{$0.019_{\pm .001}$}} & \scalea{$0.050_{\pm .000}$} & \scalea{$0.138_{\pm .014}$} & 
        \scalea{$0.358_{\pm .023}$} & \scalea{$4.864_{\pm .989}$} & \scalea{$0.043_{\pm .011}$} \\
        & \scalea{64} & 
        \bst{\scalea{$0.024_{\pm .002}$}} & \subbst{\scalea{$0.065_{\pm .002}$}} & \scalea{$0.185_{\pm .000}$} & \scalea{$0.387_{\pm .023}$} & 
        \scalea{$0.460_{\pm .021}$} & \scalea{$10.320_{\pm .487}$} & \scalea{$0.284_{\pm .053}$} \\
        & \scalea{128} & 
        \bst{\scalea{$0.122_{\pm .007}$}} & \scalea{$0.421_{\pm .034}$} & \scalea{$0.583_{\pm .000}$} & \subbst{\scalea{$0.340_{\pm .031}$}} & 
        \scalea{$1.634_{\pm .117}$} & \scalea{$14.022_{\pm 1.376}$} & \scalea{$0.853_{\pm .136}$} \\
        & \scalea{256} & 
        \subbst{\scalea{$1.183_{\pm .052}$}} & \scalea{$1.575_{\pm .076}$} & \scalea{$1.586_{\pm .000}$} & \bst{\scalea{$0.723_{\pm .057}$}} & 
        \scalea{$3.200_{\pm .134}$} & \scalea{$24.603_{\pm 1.898}$} & \scalea{$1.230_{\pm .369}$} \\
        \cmidrule(lr){2-9}
        & \scalea{\emph{Avg}} & 
        \bst{\scalea{0.339}} & \scalea{0.520} & \scalea{0.601} & \subbst{\scalea{0.397}} & 
        \scalea{1.413} & \scalea{13.452} & \scalea{0.603} \\
        \midrule

        \multirow{5}{*}{\rotatebox{90}{\scalebox{0.9}{Correlational}}}
        & \scalea{24} & 
        \subbst{\scalea{$0.123_{\pm .058}$}} & \bst{\scalea{$0.115_{\pm .036}$}} & \scalea{$0.155_{\pm .000}$} & \scalea{$0.163_{\pm .035}$} & 
        \scalea{$0.358_{\pm .026}$} & \scalea{$0.494_{\pm .008}$} & \scalea{$0.170_{\pm .077}$} \\
        & \scalea{64} & 
        \bst{\scalea{$0.068_{\pm .022}$}} & \scalea{$0.111_{\pm .032}$} & \subbst{\scalea{$0.099_{\pm .000}$}} & \scalea{$0.202_{\pm .047}$} & 
        \scalea{$0.298_{\pm .059}$} & \scalea{$0.776_{\pm .003}$} & \scalea{$0.184_{\pm .044}$} \\
        & \scalea{128} & 
        \bst{\scalea{$0.091_{\pm .027}$}} & \bst{\scalea{$0.091_{\pm .033}$}} & \subbst{\scalea{$0.106_{\pm .000}$}} & \scalea{$0.179_{\pm .028}$} & 
        \scalea{$0.330_{\pm .020}$} & \scalea{$0.862_{\pm .005}$} & \scalea{$0.250_{\pm .054}$} \\
        & \scalea{256} & 
        \scalea{$0.256_{\pm .038}$} & \subbst{\scalea{$0.194_{\pm .046}$}} & \bst{\scalea{$0.168_{\pm .000}$}} & \scalea{$0.235_{\pm .030}$} & 
        \scalea{$0.229_{\pm .047}$} & \scalea{$0.916_{\pm .005}$} & \scalea{$0.332_{\pm .054}$} \\
        \cmidrule(lr){2-9}
        & \scalea{\emph{Avg}} & 
        \scalea{0.135} & \bst{\scalea{0.128}} & \subbst{\scalea{0.132}} & \scalea{0.195} & 
        \scalea{0.304} & \scalea{0.762} & \scalea{0.234} \\
        \midrule

        \multirow{5}{*}{\rotatebox{90}{\scalebox{0.9}{Discriminative}}}
        & \scalea{24} & 
        \scalea{$0.065_{\pm .014}$} & \bst{\scalea{$0.036_{\pm .018}$}} & \subbst{\scalea{$0.043_{\pm .000}$}} & \scalea{$0.072_{\pm .012}$} & 
        \scalea{$0.391_{\pm .034}$} & \scalea{$0.432_{\pm .144}$} & \scalea{$0.064_{\pm .015}$} \\
        & \scalea{64} & 
        \bst{\scalea{$0.014_{\pm .012}$}} & \scalea{$0.182_{\pm .054}$} & \scalea{$0.268_{\pm .000}$} & \subbst{\scalea{$0.125_{\pm .015}$}} & 
        \scalea{$0.389_{\pm .026}$} & \scalea{$0.355_{\pm .135}$} & \scalea{$0.142_{\pm .064}$} \\
        & \scalea{128} & 
        \scalea{$0.256_{\pm .151}$} & \scalea{$0.455_{\pm .011}$} & \scalea{$0.489_{\pm .000}$} & \bst{\scalea{$0.160_{\pm .013}$}} & 
        \scalea{$0.442_{\pm .010}$} & \scalea{$0.473_{\pm .011}$} & \subbst{\scalea{$0.238_{\pm .058}$}} \\
        & \scalea{256} & 
        \scalea{$0.484_{\pm .010}$} & \scalea{$0.448_{\pm .026}$} & \scalea{$0.485_{\pm .000}$} & \bst{\scalea{$0.274_{\pm .098}$}} & 
        \scalea{$0.462_{\pm .016}$} & \scalea{$0.434_{\pm .151}$} & \subbst{\scalea{$0.364_{\pm .095}$}} \\
        \cmidrule(lr){2-9}
        & \scalea{\emph{Avg}} & 
        \scalea{0.205} & \scalea{0.280} & \scalea{0.321} & \bst{\scalea{0.158}} & 
        \scalea{0.421} & \scalea{0.424} & \subbst{\scalea{0.202}} \\
        \midrule

        \multirow{5}{*}{\rotatebox{90}{\scalebox{0.9}{Predictive}}}
        & \scalea{24} & 
        \bst{\scalea{$0.020_{\pm .001}$}} & \subbst{\scalea{$0.021_{\pm .001}$}} & \scalea{$0.022_{\pm .000}$} & \scalea{$0.022_{\pm .001}$} & 
        \scalea{$0.030_{\pm .002}$} & \scalea{$0.101_{\pm .008}$} & \subbst{\scalea{$0.021_{\pm .002}$}} \\
        & \scalea{64} & 
        \bst{\scalea{$0.019_{\pm .001}$}} & \subbst{\scalea{$0.020_{\pm .001}$}} & \scalea{$0.023_{\pm .000}$} & \scalea{$0.021_{\pm .001}$} & 
        \scalea{$0.026_{\pm .001}$} & \scalea{$0.216_{\pm .016}$} & \scalea{$0.022_{\pm .006}$} \\
        & \scalea{128} & 
        \bst{\scalea{$0.019_{\pm .000}$}} & \subbst{\scalea{$0.021_{\pm .001}$}} & \scalea{$0.024_{\pm .000}$} & \subbst{\scalea{$0.021_{\pm .001}$}} & 
        \scalea{$0.026_{\pm .001}$} & \scalea{$0.293_{\pm .034}$} & \scalea{$0.022_{\pm .003}$} \\
        & \scalea{256} & 
        \bst{\scalea{$0.021_{\pm .000}$}} & \scalea{$0.024_{\pm .000}$} & \scalea{$0.025_{\pm .000}$} & \bst{\scalea{$0.021_{\pm .002}$}} & 
        \scalea{$0.029_{\pm .001}$} & \scalea{$0.357_{\pm .008}$} & \subbst{\scalea{$0.022_{\pm .006}$}} \\
        \cmidrule(lr){2-9}
        & \scalea{\emph{Avg}} & 
        \bst{\scalea{0.020}} & \scalea{0.022} & \scalea{0.024} & \subbst{\scalea{0.021}} & 
        \scalea{0.028} & \scalea{0.242} & \scalea{0.022} \\
        \bottomrule
    \end{tabular}
    \end{minipage}
\end{table}

\end{document}